\documentclass{article}

\usepackage[preprint]{neurips_2025}

\usepackage[utf8]{inputenc} % allow utf-8 input
\usepackage[T1]{fontenc}    % use 8-bit T1 fonts
\usepackage{hyperref}       % hyperlinks
\usepackage{url}            % simple URL typesetting
\usepackage{booktabs}       % professional-quality tables
\usepackage{amsfonts}       % blackboard math symbols
\usepackage{nicefrac}       % compact symbols for 1/2, etc.
\usepackage{microtype}      % microtypography
\usepackage{xcolor}         % colors
\usepackage{listings}
\usepackage{todonotes}          % For margin comments
\usepackage{ifthen}
\definecolor{pastelquery}{HTML}{F5F5DC}  % Soft beige pastel for query background
\definecolor{dpoColor}{HTML}{66c2a5}     % Soft teal for DPO
\definecolor{mpoColor}{HTML}{639CD9}% Optional, if you want conditional logic
\usepackage{amssymb,amsmath,amsfonts,bm}
\def\eqref#1{equation~\ref{#1}}
\def\1{\bm{1}}

\DeclareMathAlphabet{\mathsfit}{\encodingdefault}{\sfdefault}{m}{sl}
\SetMathAlphabet{\mathsfit}{bold}{\encodingdefault}{\sfdefault}{bx}{n}

\usepackage{graphicx}
\usepackage{subfigure}
\usepackage{natbib}

\usepackage{xcolor}
\usepackage{colortbl}
\usepackage{hyperref}
\definecolor{DarkBlue}{rgb}{0.1,0.1,0.5}
\definecolor{DarkGreen}{rgb}{0.1,0.5,0.1}
\definecolor{deepyellow}{RGB}{218, 174, 42}
\hypersetup{
	colorlinks   = true,
	linkcolor    = DarkBlue, % color of internal links
	urlcolor     = DarkBlue, % color of external links
	citecolor    = DarkGreen % color of links to bibliography
}

\usepackage{xfrac}
\usepackage{amsthm}
\usepackage{thmtools,thm-restate}
\usepackage{amsmath,amsfonts,dsfont}
\usepackage[mathscr]{euscript}
\usepackage{booktabs} % for professional tables

\usepackage{multirow}
\usepackage{svg}
\usepackage{makecell}
\usepackage{pgfplots}
\pgfplotsset{compat=1.18} % adjust to the version you have installed

\definecolor{DarkBlue}{rgb}{0.3,0.3,0.70}
\definecolor{azure}{rgb}{0.0, 0.5, 1.0}
\definecolor{darkcerulean}{rgb}{0.03, 0.27, 0.49}
\definecolor{denim}{rgb}{0.08, 0.38, 0.74}
\definecolor{DarkGreen}{rgb}{0.3,0.7,0.3}
\definecolor{lighter-gray}{gray}{0.95}
\hypersetup{
	colorlinks   = true,
	linkcolor    = DarkBlue, % color of internal links
	urlcolor     = darkcerulean, % color of external links
	citecolor    = denim % color of links to bibliography
}
\definecolor{AlgHighlight}{HTML}{228B22}

\newtheoremstyle{thmstyle}
{0.5em} % Space above
{0.15em} % Space below
{} % Body font
{} % Indent amount
{\bfseries} % Theorem head font
{.} % Punctuation after theorem head
{.5em} % Space after theorem head
{} % Theorem head spec (can be left empty, meaning `normal')

\theoremstyle{thmstyle} 
\newtheorem{theorem}{Theorem}
\newtheorem{lemma}{Lemma}

\newtheorem{corollary}{Corollary}

\theoremstyle{definition}

\theoremstyle{remark}
\newtheorem{remark}{Remark}

\renewcommand{\1}{ \mathds{1}}

\newcommand{\CommentLines}[1]{}

\usepackage{array}
\newcolumntype{x}[1]{>{\centering\let\newline\\\arraybackslash\hspace{0pt}}m{#1}}
\usepackage{colortbl}

\usepackage{algorithm}
\usepackage{algorithmic}
\usepackage{wrapfig}
\usepackage{svg}

\usepackage{amsmath} % for checkmark
\usepackage{amssymb} % for checkmark
\usepackage{colortbl} % for cell colors
\usepackage{enumitem}
\usepackage{caption}

\definecolor{green}{HTML}{C6EFCE}
\definecolor{red}{HTML}{FFC7CE}
\definecolor{yellow}{HTML}{FFEB9C}

\usepackage{listings}

\definecolor{darkgreen}{rgb}{0, 0.6, 0}

\newcommand{\dpo}{\textsc{Dpo}}
\newcommand{\swepo}{\textsc{Mpo}}
\newcommand{\simpo}{\textsc{Simpo}}
\newcommand{\mpo}{\textsc{Mpo}}

\usepackage[most]{tcolorbox}     % loads xcolor automatically
\definecolor{turquoise}{HTML}{30D5C8}  % ↓ adjust if you want a different shade

\newtcolorbox{keytakeawaysbox}{%
  enhanced,
  colback=white,
  colframe=turquoise,
  coltitle=black,                    % <-- title text colour
  boxrule=0.8pt,
  arc=2mm,                           % <-- less rounded corners
  left=2mm,right=2mm,top=0.5mm,bottom=0.5mm,
  title={\small\bfseries Key Takeaways},
  attach boxed title to top left={
    yshift=-\tcboxedtitleheight/2,
    xshift=6pt
  },
  boxed title style={
    colback=white,
    colframe=turquoise,
    boxrule=0.8pt,
    arc=2mm                          % match outer radius
  },
}

\newtcolorbox{keytakeawaysbox2}{%
  enhanced, colback=white, colframe=turquoise,          % colours
  boxrule=0.6pt, arc=2mm,                               % thin border, small radius
  left=2mm,right=2mm,top=1mm,bottom=1mm,                % minimal interior padding
  before skip=0pt, after skip=0pt,                      % no extra vertical space
  title={\small\bfseries Key Takeaways},                % title text
  attach boxed title to top left={
    xshift=4pt, yshift=-2.5mm                           % “break” size
  },
  boxed title style={
    colback=white, colframe=turquoise,
    boxrule=0.6pt, arc=2mm,
    top=0pt,bottom=0pt,left=1mm,right=1mm               % tight little title tag
  },
}

\definecolor{remarkblue}{HTML}{0B5CA3} 

\newtcolorbox{remarkbox}[1]{%
  enhanced, breakable,
  colback=white,                  % body background
  colframe=remarkblue,            % border colour
  colbacktitle=remarkblue,        % title‐bar background
  coltitle=white,                 % title‐bar text colour
  boxrule=0.7pt,
  arc=2mm,                        % slightly rounded
  left=2mm,right=2mm,top=0mm,bottom=0mm,  % inner padding
  fonttitle=\bfseries,            % bold title
  title={\strut #1},              % \strut keeps the bar height uniform
  before upper=\itshape,          % body text italic
  after skip=1ex,                 % small space below the box
}

\definecolor{titlecolor}{RGB}{ 45,  82,160}% Deep Royal Blue
\definecolor{inputcolor}{RGB}{ 25, 111, 61}% Forest Green
\definecolor{mathcolor} {RGB}{112,  48,160}% Soft Purple
\definecolor{stepcolor} {RGB}{ 23,  43, 77}% Midnight Blue (for section / loop labels)
\definecolor{returncolor}{RGB}{192,  57, 43}% Pomegranate Red
\definecolor{algHlight}{HTML}{228B22}% A muted green for “important” actions
\definecolor{commentcolor}{RGB}{96,110,140}
\usepackage{setspace}

\definecolor{takeawayframe}{RGB}{45,82,160}   % frame colour
\definecolor{takeawayback} {RGB}{245,248,252} % background

\newtcolorbox{takeawaybox}[1][]{%
  enhanced, breakable,
  sharp corners=south,         % same look you used before
  colback   = gray!5!white,
  colframe  = gray!10!white,
  boxrule   = 0.7pt,
  left      = 1mm,             % inner padding (defaults)
  right     = 1mm,
  top       = 0.5mm,
  bottom    = 0.5mm,
  before skip=0pt,             % no extra vertical glue
  after skip =0pt,             % ← behaves like a normal text line
  before upper=\itshape,       % italic body
  #1}                          % allow per-box overrides

\definecolor{darkgreen}{rgb}{0, 0.6, 0}

\title{Multi-Preference Optimization: Generalizing DPO via Set-Level Contrasts}

\author{
    Taneesh Gupta\thanks{Equal contribution} \\
    \textbf{Microsoft}
  \And
    Rahul Madhavan$^\star$ \\
    \textbf{IISc, Bangalore}
  \And
    Xuchao Zhang \\
    \textbf{Microsoft}
  \AND
    Nagarajan Natarajan \\
    \textbf{Microsoft}
  \And
    Chetan Bansal \\
    \textbf{Microsoft}
  \And
    Saravan Rajmohan \\
    \textbf{Microsoft}
}

\begin{document}

\maketitle

\vspace{-0.15in}

\begin{abstract}
% Direct Preference Optimization (DPO) has proven effective in aligning large language models using pairwise preferences. In practice, as we do on-policy generation, and employ LLMs-as-reward-models, we have the ability to score multiple responses per query. In such a setting, we revisit alignment, and
% %In practice, overlooking additional positive and negative responses that are commonly available during on-policy generation. 
% propose \textbf{Multi-Preference Optimization} (\mpo), which incorporates multiple responses per query and prioritizes those that deviate most from the average reward. This deviation-based weighting focuses training on the most informative outliers, akin to a built-in curriculum. Theoretically, we prove that such multi-preference sampling lowers alignment bias, bounding the expected deviation from the true acceptable-response distribution at a rate of \(\mathcal{O}\!\bigl(\tfrac{1}{\sqrt{k}}\bigr)\). Empirically, \mpo\ outperforms state-of-the-art (SOTA) methods on the widely-used Ultra-Feedback dataset. We also demonstrate substantial improvements of up to $\sim 17.5$ \% compared to SOTA in the length-controlled win rate on AlpacaEval2.

Direct Preference Optimization ($\dpo$) has become a popular approach for aligning language models using pairwise preferences. However, in practical post-training pipelines, on-policy generation typically yields multiple candidate responses per prompt, which are scored by a reward model to guide learning. In this setting, we propose \textbf{Multi-Preference Optimization ($\mpo$)}, a generalization of $\dpo$ that optimizes over entire sets of responses by extending the Bradley--Terry model to groupwise comparisons between chosen and rejected sets. To further enhance learning, $\mpo$ employs deviation-based weighting, which emphasizes outlier responses that deviate most from the mean reward, effectively inducing a self-paced curriculum. We theoretically prove that $\mpo$ reduces alignment bias at a rate of \(\mathcal{O}\!\bigl(\tfrac{1}{\sqrt{n}}\bigr)\) with respect to the number of responses per query. Empirically, $\mpo$ achieves state-of-the-art performance on the UltraFeedback benchmark and yields up to  $\sim 17.5$\%  improvement over the state-of-the-art baseline in length-controlled win rate on AlpacaEval2, establishing a new baseline for preference-based alignment.

\end{abstract}

% \begin{figure}[!thbp]
%     \centering
%     \includegraphics[width=0.8\linewidth]{images/DPO_MPO_Comparison.pdf}
%     \vspace{-0.25in}
%     % \includegraphics[draft]{images/Sweppo_illustration.pdf}
%     \caption{In DPO, one selects a pair of responses from the base model to optimize the LLM. But this is wasteful, as there may be a lot of signal in the other responses provided by the LLM to a given query. The primary motivation of our MPO algorithm is to optimize the LLM over multiple responses from the policy model. }
%     \label{fig:dpo-mpo-comparison}
% \end{figure}

\vspace{-0.2in}
\section{Introduction}
\label{sec:introduction}

\vspace{-0.1in}
Aligning large language models (LLMs) to follow instructions effectively and adhere to desired behaviors is a cornerstone of modern AI development \citep{ouyang2022training, christiano2017deep}. Direct Preference Optimization (DPO) has emerged as a prominent and computationally efficient paradigm for the alignment task. DPO bypasses the need for explicit reward model training by directly optimizing a policy based on pairwise preferences \citep{rafailov2024direct}. However, the standard DPO framework is inherently limited to comparing single pairs of (preferred, dispreferred) responses; 
\begin{wrapfigure}{r}{0.5\linewidth}
    \centering
    \vspace{-0.2in}
    \includegraphics[width=1.0\linewidth]{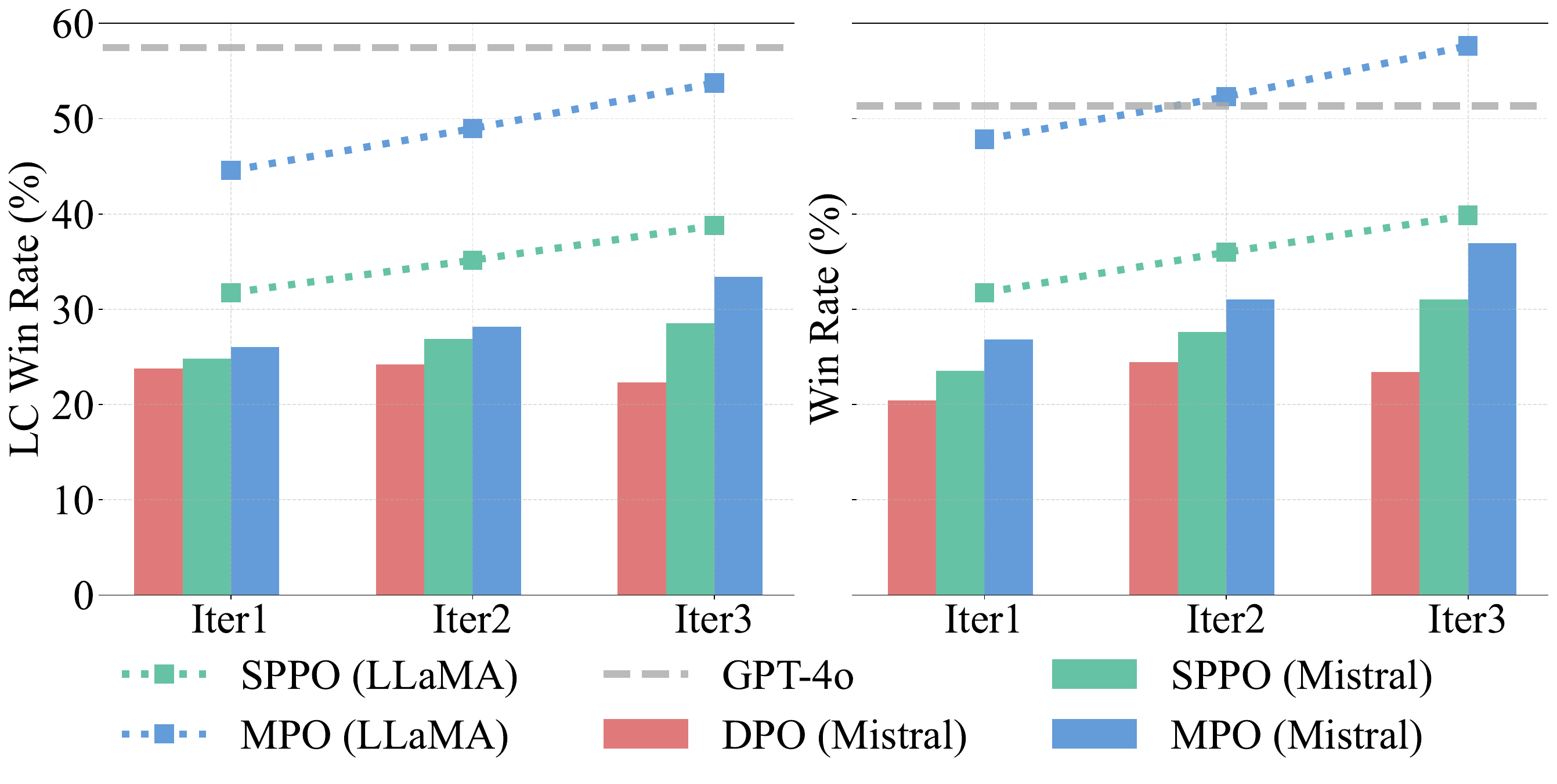}
    \vspace{-0.2in}
    \caption{Multi-Preference Optimization ($\mpo$) achieves state-of-the-art results on WR, LC-WR on AlpacaEval2, with performance nearly matching GPT-4o, demonstrating \mpo's efficacy for post-training optimization.\vspace{-0.25in}}
    \label{fig:dpo-mpo-comparison}
\end{wrapfigure}
furthermore, the standard DPO implementations use the most-preferred and the least-preferred responses per query \citep{wu2024self,meng2024simpo} discarding the rest.   

This pairwise formulation becomes restrictive as modern alignment pipelines, especially those involving on-policy generation or self-play, often produce a multitude of candidate responses per prompt \citep{wu2024self, chen2024self, pang2024iterative, tang2025game}, wasting valuable supervisory signal and computational effort. To overcome this, we propose \textbf{Multi-Preference Optimization (\mpo)}, which generalizes DPO by using a group-contrastive loss to compare entire sets of ``accepted'' versus ``rejected'' responses. This core idea extends DPO's underlying Bradley-Terry model \citep{bradley1952rank} from pairwise to groupwise comparisons.

\vspace{-0.05in}
Leveraging multiple responses for alignment is by no means new. Several alternative strategies have been proposed for post-training optimization using richer preference data. Traditional policy-gradient RLHF methods, such as Proximal Policy Optimization (PPO) \citep{schulman2017proximal}, can implicitly utilize multiple responses through a learned reward function. More recent methods explicitly target groupwise comparisons, like Group Relative Policy Optimization (GRPO) \citep{liu2024deepseek}, or aim to integrate explicit reward signals more directly, including Reward-Aware Preference Optimization (RPO) \citep{sun2025reward} and cross-entropy-based methods like InfoNCA \citep{chen2024noise}. In fact, we could apply the standard DPO loss across all possible $\binom{k}{2}$ pairs \citep{tunstall2023zephyr}, without discarding any of the $k$ preferences per query; or use ranking-based methods to leverage full or partial rankings over all $k$ responses \citep{dong2023raft, yuan2023rrhf}.

\vspace{-0.05in}
Many of these approaches, however, face significant hurdles. Methods like PPO, and RPO typically necessitate a well-calibrated reward model to provide the scalar values or advantages crucial for their optimization objectives \citep{schulman2017proximal, liu2024deepseek, sun2025reward}. In fact methods like GRPO only work with verifiable rewards \citep{liu2024deepseek}.
Similarly, InfoNCA relies on the magnitude of explicit reward scores for its target distribution \citep{chen2024noise}. The challenge of obtaining consistently accurate reward scores, particularly given the known difficulties in training stable value functions via regression \citep{farebrother2024stopregressingtrainingvalue}, can make these approaches sensitive to reward noise or mis-calibration. On the other hand, naively applying DPO to all pairs incurs significant computational overhead and potential optimization instabilities, while Plackett-Luce models demand reliable full rankings, which are often expensive and difficult to obtain \citep{chen2024softmax}, and may cause alignment brittleness.

\begin{figure}[t]
    \centering
    \includegraphics[width=\textwidth]{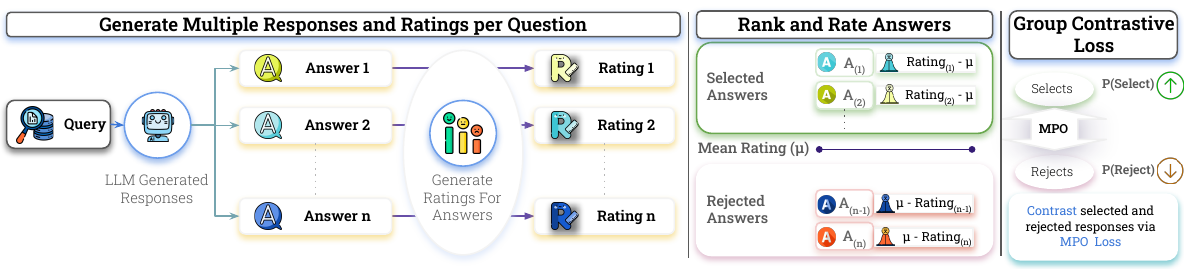}
    \vspace{-0.25in}
    % \includegraphics[draft]{images/Sweppo_illustration.pdf}
    % \caption{\mpo\ pipeline for LLM alignment. Upon generation of responses, we wish to classify responses into a selected and rejected set, either through human annotated select-reject sets, or via reward model scores for the responses. In the latter case, we calculate weights based on the absolute difference between each rating and the mean. Accepted responses (above mean) and rejected responses (below mean) are grouped separately. Our \mpo\ loss function then maximizes the probability of the accepted group while minimizing the probability of the rejected group.\vspace{-0.15in}}
    \caption{$\mpo$ pipeline: Responses are split into accepted and rejected sets using human labels or reward model scores. We assign weights based on deviation from the mean and optimize to upweight the accepted group while downweighting the rejected group using our $\mpo$ loss.\vspace{-0.25in}}
    \label{fig:weighted_preference_opt}
\end{figure}

\vspace{-0.05in}

\mpo\ is designed to overcome these specific challenges by retaining DPO's robust contrastive nature while effectively scaling to multiple preferences. It does not need well-calibrated reward models or full ranking over the responses, and instead can work with relative comparisons (ordinal preferences) over exact value regression \citep{farebrother2024stopregressingtrainingvalue}. 
When fine-grained scores are available, these sets are partitioned into ``accepted'' and ``rejected'' groups based on a query-specific score mean.
% \mpo\ operates on \textit{sets} of responses, partitioning them into ``accepted'' and ``rejected'' groups based on query-specific threshold (which is simply the mean of the scores of responses), when fine-grained scores are available. 
\mpo\ then employs a novel groupwise contrastive loss, generalizing the Bradley-Terry model \citep{bradley1952rank}, which is distinctively enhanced by a deviation-based weighting scheme.
This weighting prioritizes informative outliers, creating an implicit curriculum \citep{bengio2009curriculum} and using reward scores primarily to define groups and modulate influence.
% It employs a novel groupwise contrastive loss derived from generalizing the Bradley-Terry model \citep{bradley1952rank} to sets, enhanced by a deviation-based weighting scheme that prioritizes informative outliers, creating an implicit curriculum \citep{bengio2009curriculum}. 
% This approach avoids direct reliance on precise reward magnitudes for the loss structure, instead using rewards primarily to define the groups and weights. 
\mpo\ provides a theoretically grounded method with state-of-the-art results (Section \ref{sec:experimental_results}).\\
We provide a detailed illustration of our method in Figure \ref{fig:weighted_preference_opt}. Our primary contributions are:
% Concretely, \mpo\ is a scalable and robust method for leveraging the rich information over multi-responses, providing an upgrade for various DPO-style alignment workflows.

\vspace{-0.15in}
\paragraph{1. (Set-Based) Multi-Preference Optimization (\mpo):}
We introduce \textbf{Multi-Preference Optimization (\mpo)}, which generalizes DPO's pairwise contrast to efficiently learn from \textit{sets} of preferred and non-preferred responses per query. \mpo\ employs a novel groupwise objective based on the Bradley-Terry model, enabling richer supervision than single positive/negative pairs.
\vspace{-0.15in}
\paragraph{2. Deviation-Weighted Learning (W-\mpo) for Fine-Grained Signals:}
We propose \textbf{Weighted MPO (W-\mpo)}, an extension that incorporates fine-grained reward information by additively adjusting response logits based on their deviation from the query-specific mean reward. This prioritizes informative outliers, creating an implicit data-driven curriculum (Section \ref{ssec:weighted_mpo_main}). 

\vspace{-0.15in}
\paragraph{3. State-of-the-Art Alignment Performance:}
Both \mpo\ and W-\mpo\ achieve state-of-the-art results on diverse alignment benchmarks, including \textbf{AlpacaEval 2.0}, \textbf{Arena-Hard}, and \textbf{MT-Bench}, outperforming established baselines with various LLMs (e.g., Mistral-7B and Llama-3-8B).

% Addressing these limitations, we propose \textbf{Multi-Preference Optimization (\mpo)}. \mpo\ extends the core \textit{contrastive learning} principle of DPO, leveraging the relative robustness of comparisons over exact value regression \citep{farebrother2024stopregressingtrainingvalue}. Instead of pairwise comparisons or full rankings, \mpo\ operates on \textit{sets} of responses, partitioning them into ``accepted'' and ``rejected'' groups based on their score relative to the query-specific mean. It employs a novel groupwise contrastive loss derived from generalizing the Bradley-Terry model \citep{bradley1952rank} to sets, enhanced by a deviation-based weighting scheme that prioritizes informative outliers, creating an implicit curriculum \citep{bengio2009curriculum}. This approach avoids direct reliance on precise reward magnitudes for the loss structure, instead using rewards primarily to define the groups and weights. \mpo\ provides a theoretically grounded method (Section \ref{sec:theoretical_analysis}) with state-of-the-art results (Section \ref{sec:experimental_results}). Concretely, \mpo\ is a scalable and robust method for leveraging the rich information in multi-response data, providing an upgrade for various DPO-style alignment workflows.

% \input{includes/our_contributions}

\begin{figure*}[t]
    \vspace{-0.25in}
    \centering
    \includegraphics[width=1.0\textwidth]{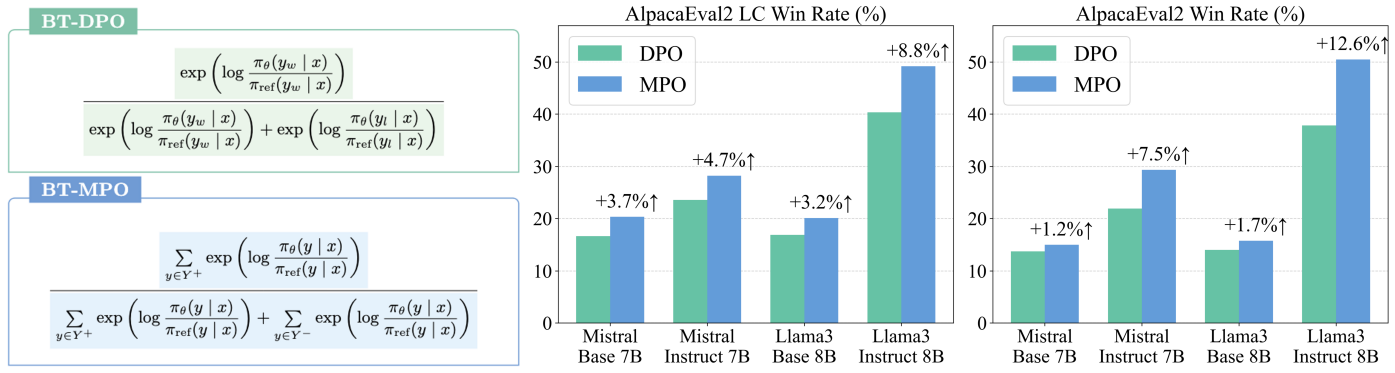}
    \vspace{-0.2in}
    \caption{\textbf{$\mpo$ generalizes $\dpo$} by extending the Bradley-Terry formulation to groupwise comparisons over multiple responses. This richer supervision yields stronger performance across model families. $\mpo$ gracefully reduces to $\dpo$ when each set contains a single response.\vspace{-0.2in}}
    \label{fig:mpo_results}
\end{figure*}

\vspace{-0.1in}
\section{Related Work}
\label{sec:related_work}

\vspace{-0.1in}
\noindent We briefly situate \mpo\ within the evolving landscape of preference optimization, deferring a more comprehensive survey to Appendix~\ref{sec:related_work_extended}. The alignment of Large Language Models (LLMs) has significantly advanced from early Reinforcement Learning from Human Feedback (RLHF) frameworks, which typically employ reward models and policy gradient methods like \textsc{Ppo}  \citep{schulman2017proximal, ouyang2022training}. A pivotal shift occurred with Direct Preference Optimization (\dpo) \citep{rafailov2024direct}, which directly learns from pairwise preferences using a contrastive loss based on the Bradley-Terry model \citep{bradley1952rank}, simplifying the alignment pipeline.
This contrastive learning paradigm, where models learn by distinguishing between positive and negative examples, has foundations in self-supervised learning, notably with objectives like InfoNCE \citep{oord2018representation} which maximize mutual information between related samples.

\vspace{-0.05in}
Many recent works extend \dpo's pairwise approach. Some focus on diversified objectives or reward modeling, such as KTO \citep{ethayarajh2024kto} and TDPO \citep{zeng2024token} for response or token-level alignment. Others modify the \dpo\ framework by altering the reference model, adding regularizers for issues like length bias, or unifying preference learning with supervised objectives, exemplified by SPIN \citep{chen2024self}, CPO \citep{xu2024contrastive}, ORPO \citep{hong2024orpo}, SimPO \citep{meng2024simpo}, R-DPO \citep{park2024disentangling}, and LD-DPO \citep{liu2024iterative}. Other works explore different structures for preference signals beyond simple pairs, such as list-wise or rank-based supervision as seen in RAFT \citep{dong2023raft} and RRHF \citep{yuan2023rrhf}.

\vspace{-0.05in}
Leveraging the multiple responses common in on-policy generation or rich datasets \citep{cui2023ultrafeedback, gupta2025ampo} necessitates moving beyond pairwise \dpo. While InfoNCA \citep{chen2024noise} adapts classification objectives for multi-response scalar rewards, \mpo\ provides a distinct paradigm by generalizing \dpo's contrastive logic to operate on \textit{sets}. This set-based comparison, especially with W-\mpo's deviation-based weighting, robustly utilizes group-level preference signals for post-training optimization.

\vspace{-0.1in}

\section{Preliminaries and Technical Background}
\label{sec:preliminaries}

\vspace{-0.1in}

\noindent We frame preference optimization as learning a policy $\pi_\theta$ based on feedback on the candidate responses $\{y_1, \dots, y_n\}$ to a query $x$. This feedback often implies a "strength" or "utility" $u(y)$ for each response. In many \dpo-style methods, this utility is commonly derived from the policy's log-probabilities relative to a reference model $\pi_{\text{ref}}$, specifically $u(y) = \exp(\beta r_\theta(y \mid x))$, where $r_\theta(y \mid x) = \log (\pi_\theta(y \mid x) / \pi_{\text{ref}}(y \mid x))$ \label{eq:implicit_score_prelim} and $\beta$ is a positive scaling factor. Different alignment methods incorporate these utilities based on how responses are grouped (or "binned") and how preferences between these responses are modeled.
Broadly, these approaches rest on a shared conceptual foundation:
% At its core, this process can be understood as follows:

\begin{remarkbox}{A Unified View of Preference Optimization}
Preference optimization can be conceptualized as a probabilistic selection process. The likelihood of favoring certain responses over others is governed by underlying "strength" or "utility" measures associated with each response. A key feature of this optimization process is that these utility measures are linked to the parameters of the language model $\pi_\theta$ being aligned.
\end{remarkbox}

\subsection{Pairwise and Setwise Preference Models (\dpo\ and \mpo)}
\label{ssec:pairwise_setwise_models}
The foundation for modeling preferences between groups of items can be traced to choice theory \citep{luce1959individual}. Consider a scenario with two distinct sets (or "bins") of responses: a preferred ("chosen") set $\mathcal{Y}^+$ and a dispreferred ("rejected") set $\mathcal{Y}^-$. If we were to select a single representative response from the combined pool $\mathcal{Y}_{\text{all}} = \mathcal{Y}^+ \cup \mathcal{Y}^-$, where the probability of selecting any specific response $y$ is proportional to its intrinsic strength or utility $u(y)$, then the probability that this selected representative comes from the preferred set $\mathcal{Y}^+$ naturally models the collective preference. We use $\mathcal{Y}^+ > \mathcal{Y}^-$ as shorthand for ``$\mathcal{Y}^+$ \text{is preferred to } $\mathcal{Y}^-$.'' The probability of this event is expressed as:
\begin{equation}
P(\mathcal{Y}^+ > \mathcal{Y}^-) = \frac{\sum_{y \in \mathcal{Y}^+} u(y)}{\sum_{y' \in \mathcal{Y}_{\text{all}}} u(y')} = \frac{\sum_{y \in \mathcal{Y}^+} u(y)}{\sum_{y \in \mathcal{Y}^+} u(y) + \sum_{y \in \mathcal{Y}^-} u(y)}.
\label{eq:set_preference_general}
\end{equation}
In the context of preference optimization and following the Bradley-Terry model \citep{bradley1952rank}, these utilities $u(y)$ are often defined as an exponential function of an underlying score or logit $s_y$, i.e., $u(y) = \exp(s_y)$. This exponential transformation ensures positivity of utilities and naturally leads to the logistic (sigmoid) function when comparing two individual items. Eq.~\ref{eq:set_preference_general} forms the conceptual basis for both \dpo and our proposed \mpo\ (See comparison in Figure \ref{fig:mpo_results}).

\vspace{-0.05in}
\paragraph{Direct Preference Optimization (\dpo).}
\dpo \citep{rafailov2024direct} is a special case of Eq.~\ref{eq:set_preference_general} where each set contains a single item: $\mathcal{Y}^+ = \{y_w\}$ (winner) and $\mathcal{Y}^- = \{y_l\}$ (loser). The \dpo objective is to maximize $\log P(y_w > y_l)$. Using $u(y) = \exp(\beta r_\theta(y|x))$ implies $\log u(y) = \beta r_\theta(y|x)$, so the standard \dpo loss becomes:

\vspace{-0.1in}
\begin{equation}
\mathcal{L}_{\text{\dpo}} = -\mathbb{E}_{(x, y_w, y_l)} \big[ \log \sigma \big( \beta (r_\theta(y_w \mid x) - r_\theta(y_l \mid x)) \big) \big].
\label{eq:dpo_loss_unified}
\end{equation}

\begin{figure*}[t]
    \vspace{-0.25in}
    \centering
    \includegraphics[width=1.0\textwidth]{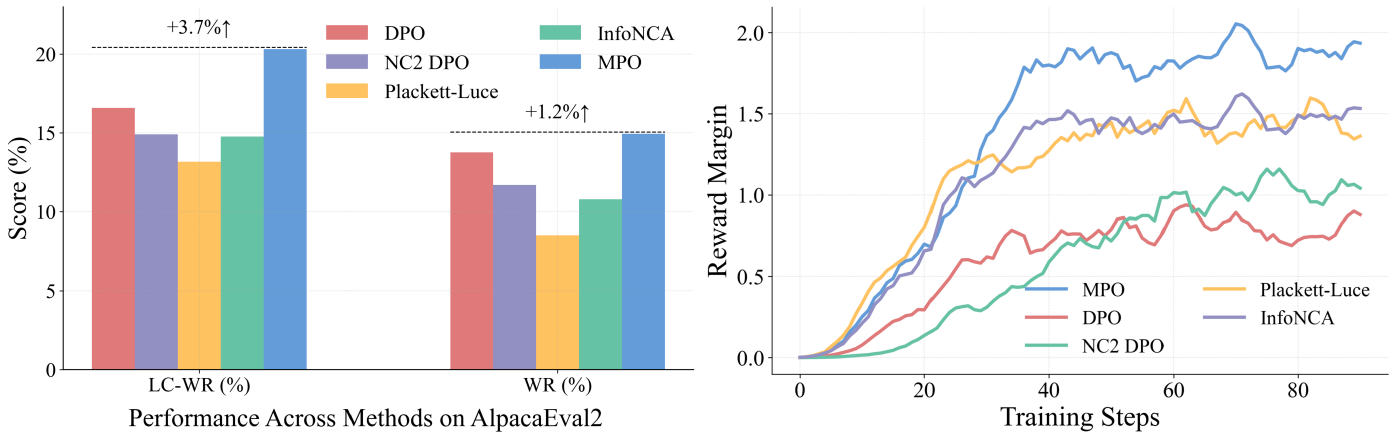}
    \vspace{-0.2in}
    \caption{$\mpo$'s consistent reward margin advantage during training on Ultrafeedback enables clearer distinction between preferred and rejected responses, which translates into downstream evaluation as well on Mistral-7B.\vspace{-0.18in}}
    \label{fig:reward-margin}
\end{figure*}

\vspace{-0.05in}
\paragraph{Multi-Preference Optimization (\mpo): A generalization of \dpo.}
Our work, \mpo, directly utilizes the setwise preference model from Eq.~\ref{eq:set_preference_general} where $\mathcal{Y}^+$ and $\mathcal{Y}^-$ can contain multiple responses. The \mpo\ objective is to maximize $\log P(\mathcal{Y}^+ > \mathcal{Y}^-)$, using the same utility definition $u(y) = \exp(\beta r_\theta(y|x))$. The specific construction of $\mathcal{Y}^+, \mathcal{Y}^-$ from data, and the introduction of weighted utilities for the W-\mpo\ variant, are detailed in Section~\ref{sec:mpo_methodology}. 

\vspace{-0.05in}
\subsection{Rank-Based and Distributional Preference Models}
\label{ssec:rank_distributional_models}

\vspace{-0.05in}
\paragraph{Plackett-Luce (PL) Model.}
When a full ranking (permutation) $y_{(1)} \succ y_{(2)} \succ \dots \succ y_{(n)}$ over $n$ items is available, the PL model \citep{plackett1975analysis} defines its probability as $P_{\text{PL}}(\text{ranking}) = \prod_{i=1}^n \frac{u(y_{(i)})}{\sum_{j=i}^n u(y_{(j)})}$, where $u(\cdot)$ represents item utilities. This can be conceptualized as sequentially selecting items for $n$ ordered "bins" (ranks), where the probability of selecting an item for the current rank depends on the utilities of the remaining unranked items. PL thus models fine-grained ordinal relationships, contrasting with the binary set partition used in \mpo.

\vspace{-0.05in}
\paragraph{Distribution Matching (e.g., InfoNCA).}
InfoNCA \citep{chen2024noise} and similar methods align $\pi_\theta$ by matching its output distribution (derived from utilities $u(y_i)$) to a target distribution. For $n$ responses $\{y_i\}$ with scalar quality scores $S_i$, a target probability $p_{\text{target}}(y_i) = \text{softmax}_i(S_i/\tau)$ is often formed. The model's predicted probability for $y_i$ is $p_{\text{model}}(y_i) = \text{softmax}_i(\log u(y_i))$. The objective then is to minimize the cross-entropy:

\vspace{-0.1in}
\begin{equation}
\mathcal{L}_{\text{InfoNCA}} \propto -\sum_{i=1}^{n} p_{\text{target}}(y_i) \log p_{\text{model}}(y_i).
\label{eq:infonca_loss_unified}
\end{equation}

\vspace{-0.05in}
This aims to replicate a detailed reward distribution over individual responses (each potentially a distinct "bin" with its own target weight). 
We provide the reward margin, a training metric, as well as the downstream performance on AlpacaEval2 for these different related post-training optimization methods in Figure \ref{fig:reward-margin}. The higher reward margin shows that the distinction between the top chosen and rejected response is higher in MPO vs. other baselines.

\section{\mpo: Multi-Preference Optimization}
\label{sec:mpo_methodology}

\begin{wrapfigure}{r}{0.52\textwidth}  % r = right-hand side
  \vspace{-0.27in}
  \begin{minipage}{\linewidth}
    \begin{algorithm}[H]
    % \setstretch{1.25}
      \caption{\textcolor{titlecolor}{\textbf{Multi-Preference Optimization (\textsc{MPO})}}}
      \label{alg:core_mpo_expanded}
      \begin{algorithmic}[1]
      % -------- Input --------------------------------------------------
      \STATE \textcolor{inputcolor}{\textbf{Input:}}
             \textcolor{inputcolor}{Initial policy $\theta_0$, 
             $\mathcal{D}_{\text{MPO}} = \{(x,\mathcal{Y}_x^+,\mathcal{Y}_x^-)\}$,
             Reference policy $\pi_{\text{ref}}$, temperature $\beta$, learning-rate $\eta$}
      \STATE \textcolor{algHlight}{Initialise $\theta \gets \theta_0$}
      \FOR{\textcolor{stepcolor}{each training iteration}}
          \FORALL{\textcolor{stepcolor}{$(x,\mathcal{Y}_x^+,\mathcal{Y}_x^-)\in\mathcal{D}_{\text{MPO}}$}}
              \FOR{\textcolor{stepcolor}{$y\in\mathcal{Y}_x^+\cup\mathcal{Y}_x^-$}}
                  \STATE \textcolor{mathcolor}{$r_\theta(y\mid x) \gets
                         \log\!\frac{\pi_\theta(y\mid x)}{\pi_{\text{ref}}(y\mid x)}$}
              \ENDFOR
              \STATE \textcolor{mathcolor}{$S_x^+ \gets
       \sum_{y\in\mathcal{Y}_x^+} \exp\!\bigl(\beta\,r_\theta(y\mid x)\bigr)$}
              \STATE \textcolor{commentcolor}{\(\triangleright\)\,Selected responses}
              \STATE \textcolor{mathcolor}{$S_x^{\mathrm{all}}\gets
                        \sum_{y\in\mathcal{Y}_x^+\cup\mathcal{Y}_x^-}
                        \exp\bigl(\beta\,r_\theta(y\mid x)\bigr)$}
              \STATE \textcolor{mathcolor}{$L_x\gets-\log\!\bigl(S_x^+/S_x^{\mathrm{all}}\bigr)$}
              \STATE \textcolor{algHlight}{Compute $\nabla_\theta L_x$}
              \STATE \textcolor{algHlight}{$\theta\gets\theta-\eta\,\nabla_\theta L_x$}
          \ENDFOR
      \ENDFOR
      \RETURN \textcolor{returncolor}{$\theta$}
  \end{algorithmic}
    \end{algorithm}
  \end{minipage}
  \vspace{-0.4in}
\end{wrapfigure}

Building upon the principles of groupwise preference modeling discussed in Section \ref{sec:preliminaries}, we introduce Multi-Preference Optimization (\mpo), a method designed to align language models using multiple preferred and dispreferred responses per query. \mpo{} generalizes DPO's pairwise comparison to a groupwise contrast using multiple preferred/dispreferred responses per query, enabling richer supervision. We define the \mpo{} objective and its weighted W-\mpo\ variant, which incorporates fine-grained reward signals.

\vspace{-0.1in}
\subsection{Problem Setup for \mpo}
\label{ssec:notation_setup_mpo}

\vspace{-0.05in}
For the \mpo\ framework, we operate on a dataset $\mathcal{D}_{\text{MPO}}$. Each instance in this dataset consists of a query $x$, a set of $n_x^+$ chosen (preferred) responses $\mathcal{Y}_x^+ = \{y_{c,1}, \dots, y_{c,n_x^+}\}$, and a disjoint set of $n_x^-$ rejected (dispreferred) responses $\mathcal{Y}_x^- = \{y_{r,1}, \dots, y_{r,n_x^-}\}$. 
These sets $\mathcal{Y}_x^+$ and $\mathcal{Y}_x^-$ can be sourced from direct annotations or derived from finer-grained scalar reward scores, as elaborated for the W-\mpo\ variant (Section~\ref{ssec:weighted_mpo_main}). 
Our objective is to fine-tune the policy $\pi_\theta(y \mid x)$ using the reference model $\pi_{\text{ref}}(y \mid x)$ to define the implicit log-preference scores $r_\theta(y \mid x)$ (as introduced in Preliminaries, Eq.~\ref{eq:implicit_score_prelim}), which are scaled by an inverse temperature $\beta > 0$.

\vspace{-0.05in}
\subsection{Base \mpo{} Objective}
\label{ssec:base_mpo_objective}

\vspace{-0.05in}
%While DPO uses response pairs, many datasets (and on-policy optimization methods) provide sets $\mathcal{Y}_x^+$ (preferred) and $\mathcal{Y}_x^-$ (dispreferred) per query. 
\mpo{} leverages sets of preferences by modeling the probability $P(\mathcal{Y}_x^+ > \mathcal{Y}_x^- \mid x; \theta, \beta)$ using the groupwise choice principle (Eq.~\ref{eq:set_preference_general} from Preliminaries). For \mpo, the utility $u(y)$ of a response $y$ is defined as $u(y) = \exp(\beta \cdot r_\theta(y \mid x))$, where $r_\theta(y \mid x)$ is the implicit log-preference score (Eq.~\ref{eq:implicit_score_prelim}). Substituting this utility definition into the general groupwise model (Eq.~\ref{eq:set_preference_general}) directly yields the MPO preference probability:

\vspace{-0.1in}
\begin{equation}
P(\mathcal{Y}_x^+ > \mathcal{Y}_x^- \mid x; \theta, \beta) = \frac{\sum_{y \in \mathcal{Y}_x^+} \exp(\beta \cdot r_\theta(y \mid x))}{\sum_{y' \in \mathcal{Y}_x^+ \cup \mathcal{Y}_x^-} \exp(\beta \cdot r_\theta(y' \mid x))}.
\label{eq:mpo_set_probability}
\end{equation}

\vspace{-0.05in}
The \mpo{} objective is to maximize this probability by minimizing its negative log-likelihood:

\vspace{-0.1in}
\begin{equation}
\mathcal{L}_{\text{MPO}}(\theta) = -\mathbb{E}_{(x, \mathcal{Y}_x^+, \mathcal{Y}_x^-) \sim \mathcal{D}_{\text{MPO}}} \left[ \log P(\mathcal{Y}_x^+ > \mathcal{Y}_x^- \mid x; \theta, \beta) \right].
\label{eq:mpo_loss_definition}
\end{equation}
Intuitively, this loss function encourages the policy $\pi_\theta$ to assign collectively higher implicit preference scores $r_\theta(y \mid x)$ (and thus higher probabilities $\pi_\theta(y \mid x)$ relative to the reference policy $\pi_{\text{ref}}$) to the responses in the preferred set $\mathcal{Y}_x^+$ compared to the combined set of preferred and dispreferred responses $\mathcal{Y}_x^+ \cup \mathcal{Y}_x^-$.

\vspace{-0.05in}
\subsection{Algorithm for \mpo}
\label{ssec:algorithm_core_mpo}

\vspace{-0.05in}
The training procedure for the core \mpo{} objective is summarized in Algorithm \ref{alg:core_mpo_expanded}. The algorithm iterates through the preference dataset $\mathcal{D}_{\text{MPO}}$. For each data sample, comprising a query $x$ with its associated preferred ($\mathcal{Y}_x^+$) and dispreferred ($\mathcal{Y}_x^-$) response sets, it computes the implicit scores, then the \mpo{} loss term, and updates the policy model $\pi_\theta$ using gradient descent.

% \begin{algorithm}[!ht]
%    \caption{Multi-Preference Optimization (\mpo)}
%    \label{alg:core_mpo_expanded}
%    \begin{algorithmic}[1]
%        \STATE {\bfseries Input:} Initial policy parameters $\theta_0$, dataset $\mathcal{D}_{\text{MPO}}$ (samples: $(x, \mathcal{Y}_x^+, \mathcal{Y}_x^-)$), reference model $\pi_{\text{ref}}$, temperature $\beta$, learning rate $\eta$
%        \STATE Initialize $\theta \leftarrow \theta_0$.
%        \FOR{each training iteration}
%            \FORALL{$(x, \mathcal{Y}_x^+, \mathcal{Y}_x^-)$ in $\mathcal{D}_{\text{MPO}}$ (or a mini-batch)}
%                \STATE For $y \in \mathcal{Y}_x^+ \cup \mathcal{Y}_x^-$:
%                \STATE \hspace{1em} $r_\theta(y \mid x) \gets \log \frac{\pi_\theta(y \mid x)}{\pi_{\text{ref}}(y \mid x)}$. \COMMENT{Ref. Eq. \ref{eq:implicit_score_prelim}}
               
%                \STATE $S_x^+ \gets \sum_{y \in \mathcal{Y}_x^+} \exp(\beta \cdot r_\theta(y \mid x))$. \COMMENT{Sum of preferred strengths}
%                \STATE $S_x^{\text{all}} \gets \sum_{y' \in \mathcal{Y}_x^+ \cup \mathcal{Y}_x^-} \exp(\beta \cdot r_\theta(y' \mid x))$. \COMMENT{Sum of all strengths}
               
%                \STATE $L_x \gets -\log (S_x^+ / S_x^{\text{all}})$. \COMMENT{Sample MPO loss, ref. Eq. \ref{eq:mpo_loss_definition}}
               
%                \STATE Compute gradient $\nabla_\theta L_x$.
%                \STATE Update $\theta \leftarrow \theta - \eta \nabla_\theta L_x$. \COMMENT{Or accumulate for batch update}
%            \ENDFOR
%        \ENDFOR
%        \RETURN $\theta$.
%    \end{algorithmic}
% \end{algorithm}

\subsection{Weighted \mpo{} (W-\mpo) Variant}
\label{ssec:weighted_mpo_main}

The MPO objective treats all responses within preferred ($\mathcal{Y}_x^+$) or dispreferred ($\mathcal{Y}_x^-$) sets uniformly. However, real-world preference data, often using scalar quality scores $S(y)$, exhibit significant intra-set quality variations. Weighted \mpo{} (W-\mpo) is designed to leverage this finer-grained information. By modulating each response's influence based on its deviation from the average quality, W-\mpo{} prioritizes "outlier" responses, creating a data-driven curriculum focused on the most informative examples (further motivation in Appendix~\ref{app:weighted_mpo_details}; results in Table~\ref{tab:base-results}).

For W-\mpo, we assume query $x$ has $n_x$ responses $\mathcal{Y}_x = \{y_1, \dots, y_{n_x}\}$ with corresponding scalar quality scores $\{S(y_i)\}_{i=1}^{n_x}$. 
% \mytodo{Notation issue: we use $k$ for number of responses sometimes, $n$ sometimes, and here $N$.} Fixed.
First, we compute the mean score for the query $S_{\text{mean}}(x) = \sum_{i=1}^{n_x} S(y_i)/n_x$.
% \vspace{-0.15in}
% \begin{equation}
% S_{\text{mean}}(x) = \frac{1}{N_x} \sum_{i=1}^{N_x} S(y_i).
% \label{eq:main_mean_score_wmpo}
% \end{equation}
Responses are then partitioned into a preferred set $\mathcal{Y}_x^+ = \{y \in \mathcal{Y}_x \mid S(y) > S_{\text{mean}}(x)\}$ and a dispreferred set $\mathcal{Y}_x^- = \{y \in \mathcal{Y}_x \mid S(y) \le S_{\text{mean}}(x)\}$. (If $\mathcal{Y}_x^+$ is empty, the sample may be skipped). The \textit{absolute deviation} of each response's score $S(y)$ from this mean is:
\begin{equation}
\Delta W_{\text{abs}}(y) = |S(y) - S_{\text{mean}}(x)|.
\label{eq:main_delta_W_abs_displayed}
\end{equation}
W-\mpo{} then modifies the logit for each response $y \in \mathcal{Y}_x$ by additively incorporating this absolute deviation:
\begin{equation}
r'_{\theta}(y \mid x; \alpha_w) = \beta \cdot r_\theta(y \mid x) + \alpha_w \cdot \Delta W_{\text{abs}}(y),
\label{eq:main_wmpo_additive_logit_abs_displayed}
\end{equation}
where $r_\theta(y \mid x)$ is the base implicit score (Eq.~\ref{eq:implicit_score_prelim}) and $\alpha_w > 0$ is a hyperparameter scaling the deviation's impact.

The additive term $\alpha_w \cdot \Delta W_{\text{abs}}(y)$ boosts the logit of outlier responses (those with $S(y)$ far from $S_{\text{mean}}(x)$), thereby increasing their influence. While the base term $\beta \cdot r_\theta(y \mid x)$ and the MPO set-partitioning drive the directional preference (favoring $\mathcal{Y}_x^+$), the absolute deviation term ensures these informative outliers more strongly affect the softmax normalization in the loss.

The W-\mpo{} loss function then uses these modified logits within the MPO framework:
\begin{equation}
\mathcal{L}_{\text{W-\mpo}}(\theta) = -\mathbb{E}_{(x, \mathcal{Y}_x, \{S(y)\})} \left[ \log \left( \frac{\sum_{y \in \mathcal{Y}_x^+} \exp(r'_{\theta}(y \mid x; \alpha_w))}{\sum_{y' \in \mathcal{Y}_x^+ \cup \mathcal{Y}_x^-} \exp(r'_{\theta}(y' \mid x; \alpha_w))} \right) \right].
\label{eq:main_wmpo_loss_abs_displayed}
\end{equation}
The expectation is over the original responses and their scores, from which $\mathcal{Y}_x^+$ and $\mathcal{Y}_x^-$ are derived.

\begin{remarkbox}{W-\mpo\ Robustness via Weighted Contrastive Learning}
\label{rem:comp_with_advantage_function}
Unlike methods like GRPO \citep{shao2024deepseekmath} that may use signed advantages directly as update targets (ideal with verified rule-based rewards 
% \mytodo{rule-based rewards?} Fixed
), W-\mpo's additive \textit{absolute} deviation weighting (Eq.~\ref{eq:main_wmpo_additive_logit_abs_displayed}) is exponentiated within \mpo's groupwise contrastive loss (Eq.~\ref{eq:main_wmpo_loss_abs_displayed}). This optimization is outlier driven and is robust to noise close to median responses, as our experiments demonstrate.
\end{remarkbox}

% \begin{remarkbox}{Comparison with Advantage Weighting Based Methods like GRPO}
% W-\mpo's use of an additive \textit{absolute} deviation (Eq. \ref{eq:main_wmpo_additive_logit_abs}) contrasts with methods like Group Relative Policy Optimization (GRPO). GRPO may use signed, normalized ``advantages'' as direct loss multipliers or reward targets. Instead, W-\mpo\ amplifies the logits of all outliers (high or low scoring) within its groupwise contrastive loss (Eq. \ref{eq:main_wmpo_loss_abs}), distinct from directional advantage re-weighting. While GRPO works well in RLVR settings with verified rewards, MPO is robust to reward model scoring as shown through our experiments (Sec. \ref{sec:experimental_results})
% \end{remarkbox}

The algorithmic implementation of W-\mpo{} involves calculating these modified logits $r'_{\theta}(y \mid x; \alpha_w)$ for each response before computing the full loss term, as detailed in Appendix \ref{app_ssec:algorithm_wmpo} (Algorithm \ref{alg:appendix_wmpo_detailed}).

\vspace{-0.1in}
\subsection{Theoretical Analysis}
\label{ssec:theoretical_analysis}

% \noindent Our \mpo\ framework is supported by theoretical insights into its optimization dynamics and alignment properties. We briefly present two key results here, with full derivations and discussions deferred to the Appendix.

\vspace{-0.07in}
\paragraph{1. Bias Reduction with Multiple Preferences.}
Incorporating multiple preferred ($\mathcal{Y}^+_x$) and non-preferred ($\mathcal{Y}^-_x$) responses per query $x$ systematically reduces alignment bias. Alignment bias, in this context, refers to the deviation between the expected characteristics (e.g., helpfulness, adherence to safety guidelines) learned by the model from a finite sample of $n$ responses, and the true characteristics desired from the overall distribution of acceptable responses.
% Alignment bias $B^{(k)}$ can be defined as the deviation of the model's expected value for some response attribute $a(y)$ (e.g., helpfulness, length) from the true expected attribute value over the actual distribution of acceptable responses, when $k = |\mathcal{Y}^+_x|$ samples are used.

\begin{takeawaybox}
\begin{theorem}[\textbf{Bias Reduction Rate}]
\label{thm:main_bias_reduction}
Under standard assumptions (including finite variance of attributes, independent sampling, sufficient model capacity, and uniform bounded variance across queries), the expected alignment bias $\mathbb{E}[B^{(n)}]$ when using $n$ preferred samples per query decreases as: \vspace{-0.1in}
\begin{equation}
\mathbb{E}[B^{(n)}] = \tilde{\mathcal{O}} \left(\frac{C}{\sqrt{n}}\right),
\label{eq:main_bias_reduction_rate}
\end{equation}
where $C$ is a constant related to the maximum attribute variance.
\end{theorem}
\end{takeawaybox}

% \vspace{-0.05in}
This result theoretically grounds the benefit of leveraging multiple preferences, as \mpo\ does, over single pairwise comparisons. Detailed assumptions and proof are presented in Appendix~\ref{sec:bias_analysis}.

\vspace{-0.07in}
\paragraph{2. Stationary Point of the \mpo\ Loss.}
A crucial property of the \mpo\ loss function (Eq.~\ref{eq:mpo_loss_definition}) concerns its optimization landscape. We find that at its stationary points, the model learns to assign vanishing probabilities to non-preferred responses.

% \vspace{-0.05in}
\begin{takeawaybox}
\begin{lemma}[\textbf{Vanishing Probability for Non-Preferred Responses}]
\label{lemma:main_stationary_point}
For the \mpo\ loss function, a stationary point is achieved as the model probabilities for all responses $y_i$ in the non-preferred set $\mathcal{Y}^-$ approach zero, i.e., $P_\theta(y_i \mid x) \to 0$ for all $y_i \in \mathcal{Y}^-$.
\end{lemma}
\end{takeawaybox}

% \vspace{-0.05in}
This contrasts with objectives like InfoNCA that aim to match a target distribution over all responses. This property encourages \mpo\ to learn a clearer distinction between preferred and non-preferred sets. Detailed proof and characterization are provided in Appendix~\ref{sec:stationary_point_characterization}.

\vspace{-0.15in}
\section{Experimental Setup}
\label{sec:experimental setup}
\vspace{-0.15in}
We evaluate \mpo{} across three training regimes: Offline, Online, and Iterative, using Mistral-7B and Llama-3-8B base models. Further details, including baselines, are in Appendix~\ref{app:experimental_details}. 

\vspace{-0.35cm}

\paragraph{Off-policy Setting.}
Following \citet{tunstall2023zephyr}, we first create supervised fine-tuned (SFT) models using UltraChat-200k \citep{ding2023enhancing}. These SFT models initialize \mpo\ and other methods for preference optimization on the UltraFeedback dataset \citep{cui2023ultrafeedback}.

% \paragraph{\textbf{Offline Setting}} For the offline experiments, we follow a training pipeline inspired by Zephyr \cite{tunstall2023zephyr}. We begin by training a base model—either Mistral-7B or Llama-3-8B using the UltraChat-200k \cite{ding2023enhancing} dataset to obtain a supervised fine-tuned (SFT) model. This SFT checkpoint then serves as the initialization for $\mpo$ and other offline preference optimization methods on the UltraFeedback \cite{cui2023ultrafeedback} dataset.

\vspace{-0.35cm}

% \begin{keytakeawaysbox}
% Our results show that 
% % \begin{itemize}[left=0pt,itemsep=0pt]
% %   \item 
% % \end{itemize}
% \end{keytakeawaysbox}

% \paragraph{\textbf{Online Setting}} In the online setting, we start from strong, instruction-tuned models—specifically Llama-3-8B-Instruct and Mistral-7B-Instruct which have already undergone extensive alignment through instruction tuning. To better match the preference optimization distribution to these instruction-tuned checkpoints, we generate the preference datasets using the same models, following a strategy similar to prior work \cite{meng2024simpo}. For each prompt from UltraFeedback, we sample five candidate responses using a temperature of 1.0. These are then scored by the Skywork-Reward-Llama-3.1-8B reward model, from which we select the top-2 and bottom-2 responses to serve as preferred and rejected samples.

\paragraph{On-policy Setting.}
We start with instruction-tuned models (Llama-3-8B-Instruct, Mistral-7B-Instruct). To align these models on preferences, we generate 5 responses per UltraFeedback prompt (temperature 1.0) using the instruction-tuned models themselves, similar to \citet{meng2024simpo}. Responses are scored by Skywork-Reward-Llama-3.1-8B \cite{liu2024skywork}; the top-2 and bottom-2 are selected as preferred ($\mathcal{Y}^+$) and rejected ($\mathcal{Y}^-$) sets for $\mpo$.

\begin{table*}[!t]
\centering
\resizebox{1.0\textwidth}{!}{
\begin{tabular}{@{}lcccccccc@{}}
\toprule
\multirow{2}{*}{\textbf{Method}} & \multicolumn{4}{c}{\textbf{Mistral-Base (7B)}} & \multicolumn{4}{c}{\textbf{Llama-3-Base (8B)}} \\ 
\cmidrule(lr){2-5} \cmidrule(lr){6-9}
 & LC (\%) & WR (\%) & Arena-Hard & MT-Bench & LC (\%) & WR (\%) & Arena-Hard & MT-Bench \\ 
\midrule
SFT\footnotemark[1] & 8.4 & 6.2 & 1.3 & 6.3 & 6.2 & 4.6 & 3.3 & 6.6 \\
RRHF\footnotemark[1] & 11.6 & 10.2 & 5.8 & 6.7 & 12.1 & 10.1 & 6.3 & 7.0 \\
SLiC-HF\footnotemark[1] & 10.9 & 8.9 & \underline{7.3} & \textbf{7.4} & 12.3 & 13.7 & 6.0 & 7.6 \\
IPO\footnotemark[1] & 10.8 & 9.0 & 7.5 & 7.0 & 14.4 & 14.2 & 17.8 & 7.4 \\
CPO\footnotemark[1] & 9.8 & 8.9 & 6.7 & 6.9 & 10.8 & 8.1 & 5.8 & 7.4 \\
KTO\footnotemark[1] & 13.1 & 9.1 & 5.6 & 6.8 & 14.2 & 12.4 & 12.5 & \textbf{7.8} \\
ORPO\footnotemark[1] & 14.7 & 12.2 & 7.0 & 7.0 & 12.2 & 10.6 & 10.8 & 7.6 \\
R-DPO\footnotemark[1] & 17.4 & 12.8 & 10.5 & 7.0 & 17.6 & 14.4 & 17.2 & 7.5 \\
GRPO & 10.1 & 6.7 & 6.3 & 6.8 & 8.2 & 5.7 & 6.0 & 7.1 \\
$\dpo$& 16.6 & 13.8 & 12.7 & 6.7 & 16.9 & 14.1 & \textbf{18.5} & \underline{7.7} \\
NC2-$\dpo$ & 14.9 & 11.7 & 8.8 & 6.9 & 16.3 & 13.1 & 13.7 & 7.5 \\
InfoNCA & 14.8 & 10.8 & 9.7 & 7.0 & 15.9 & 12.9 & 14.8 & 7.6 \\
\midrule
$\mpo$ 1vsk & 15.2 & 11.5 & 10.1 & 7.1 & 17.3 & 13.5 & 15.9 & 7.6 \\
$\mpo$ & \underline{18.4} & \underline{14.4} & \textbf{13.2} & 7.2 & 18.4 & 15.1 & \underline{18.4} & 7.5 \\
W-$\mpo$ & \textbf{20.3} & \textbf{14.9} & \underline{12.8} & \underline{7.3} & \textbf{20.1} & \textbf{15.6} & \textbf{18.5} & \textbf{7.8} \\
\bottomrule
\end{tabular}
}
\caption{Comparison of methods on Mistral-Base and Llama-Base models in the off-policy setting with Ultrafeedback prompts and responses.\vspace{-0.15in}}
\label{tab:base-results}
\end{table*}
\footnotetext[1]{These are taken directly from the paper \simpo: Simple Preference Optimization with a Reference-Free Reward \citep{meng2024simpo}}

\vspace{-0.35cm}

\paragraph{\textbf{Iterative Setting}} 
Using the same instruction-tuned base models, we adopt an iterative alignment framework akin to SPPO \citep{wu2024self}. The UltraFeedback prompt set is split for multi-round fine-tuning, with fresh responses (5 per prompt, temp 1.0, scored by Skywork-Reward-Llama-3.1-8B, top-2/bottom-2 selected) generated each round, following \citet{wu2024self}.

We provide more details about our experimental setup in Appendix \ref{sec:baselines} and \ref{app:experimental_details}.

\begin{table*}[t]
\centering
\resizebox{1.0\textwidth}{!}{
\begin{tabular}{@{}lcccccccc@{}}
\toprule
\multirow{2}{*}{\textbf{Method}} & \multicolumn{4}{c}{\textbf{Mistral-Instruct (7B)}} & \multicolumn{4}{c}{\textbf{Llama-3-Instruct (8B)}} \\ 
\cmidrule(lr){2-5} \cmidrule(lr){6-9}
 & LC (\%) & WR (\%) & Arena-Hard & MT-Bench & LC (\%) & WR (\%) & Arena-Hard & MT-Bench \\ 
\midrule
SFT & 17.1 & 14.7 & 12.6 & 7.5 & 26.0 & 25.3 & 22.3 & 8.1 \\
RRHF\footnotemark[1] & 25.3 & 24.8 & 18.1 & 7.6 & 31.3 & 28.4 & 26.5 & 7.9 \\
SLiC-HF\footnotemark[1] & 24.1 & 24.6 & 18.9 & \textbf{7.8} & 26.9 & 27.5 & 26.2 & 8.1 \\
DPO\footnotemark[1] & 26.8 & 24.9 & 16.3 & 7.6 & 40.3 & 37.9 & 32.6 & 8.0 \\
IPO\footnotemark[1] & 20.3 & 20.3 & 16.2 & \textbf{7.8} & 35.6 & 35.6 & 30.5 & \textbf{8.3} \\
CPO\footnotemark[1] & 23.8 & 28.8 & 22.6 & 7.5 & 28.9 & 32.2 & 28.8 & 8.0 \\
KTO\footnotemark[1] & 24.5 & 23.6 & 17.9 & 7.7 & 33.1 & 31.8 & 26.4 & 8.2 \\
ORPO\footnotemark[1] & 24.5 & 24.9 & 20.8 & 7.7 & 28.5 & 27.4 & 25.8 & 8.0 \\
R-DPO\footnotemark[1] & 27.3 & 24.5 & 16.1 & 7.5 & 41.1 & 37.8 & 33.1 & 8.0 \\
$\mpo$ & \textbf{28.2} & \textbf{29.4} & \textbf{22.7} & 7.7 & \textbf{49.0} & \textbf{50.6} & \textbf{46.2} & 8.2 \\ 
\bottomrule
\end{tabular}
}
\caption{Results on Mistral-Instruct and Llama-Instruct models in the on-policy setting with Ultrafeedback prompts.}
\label{tab:instruct-results}
\end{table*}

\begin{table*}[!tbph]
\centering
\resizebox{1.0\textwidth}{!}{
\begin{tabular}{@{}lcccccccc@{}}
\toprule
\multirow{2}{*}{\textbf{Method}} & \multicolumn{4}{c}{\textbf{Mistral-Instruct (7B)}} & \multicolumn{4}{c}{\textbf{Llama-3-Instruct (8B)}} \\ 
\cmidrule(lr){2-5} \cmidrule(lr){6-9}
 & LC (\%) & WR (\%) & Arena-Hard & MT-Bench & LC (\%) & WR (\%) & Arena-Hard & MT-Bench \\ 
\midrule
% \multicolumn{9}{c}{\textbf{Iterative Setting}} \\
% \midrule
SPPO\footnotemark[2] (Iteration 1) & 24.79 & 23.51 & 18.7 & 7.21 & 31.73 & 31.74 & 34.8 & 8.1 \\
$\mpo$ (Iteration 1) & \textbf{26.01} & \textbf{26.81} & \textbf{22.1} & \textbf{7.64} & \textbf{44.56} & \textbf{47.83} & \textbf{42.8} & \textbf{8.2} \\
\midrule
SPPO\footnotemark[2] (Iteration 2) & 26.89 & 27.62 & 20.4 & 7.49 & 35.15 & 35.98 & 38.0 & 8.2 \\
$\mpo$ (Iteration 2) & \textbf{28.15} & \textbf{31.05} & \textbf{23.3} & \textbf{7.68} & \textbf{48.96} & \textbf{52.3} & \textbf{45.6} & \textbf{8.3} \\
\midrule
SPPO\footnotemark[2] (Iteration 3) & 28.53 & 31.02 & 23.3 & 7.59 & 38.78 & 39.85 & 40.1 & 8.2 \\
$\mpo$ (Iteration 3) & \textbf{33.42} & \textbf{36.93} & \textbf{23.9} & \textbf{7.71} & \textbf{53.75} & \textbf{57.65} & \textbf{49.2} & \textbf{8.4} \\
\bottomrule
\end{tabular}
}
\caption{Results on Mistral-Instruct and Llama-Instruct models in the iterative on-policy setting with
Ultrafeedback prompts.\vspace{-0.23in}}
\label{tab:iterative}
\end{table*}
\footnotetext[2]{These are taken directly from the paper SPPO: Self-Play Preference Optimization for Language Model Alignment
 \citep{wu2024self}}

\section{Experimental Results}
\label{sec:experimental_results}

% \vspace{-0.12in}
\paragraph{SOTA Performance}
We evaluate $\mpo$ against a broad spectrum of preference optimization baselines in three distinct training regimes: offline, online, and iterative. Across all settings, $\mpo$ consistently sets a new state-of-the-art. In the offline setup, as seen in Table~\ref{tab:base-results}, $\mpo$ outperforms established baselines such as R-DPO \citep{rafailov2024direct}, InfoNCA \citep{chen2024noise}, and KTO \citep{ethayarajh2024kto}. Notably, W-$\mpo$ achieves top scores.

% \vspace{-0.03in}
In the online setup in Table~\ref{tab:instruct-results}, preference optimization is performed on instruction-tuned models. $\mpo$ again achieves SOTA performance across all benchmarks and model families. We also evaluate $\mpo$ in the iterative setting in Table~\ref{tab:iterative}. Compared against SPPO \citep{wu2024self}, the strongest prior iterative baseline, $\mpo$ yields significant improvements in every round. 

\begin{figure}[!h]               % one float, two columns inside
  \centering
  \begin{minipage}[t]{0.55\textwidth}
  % \vspace{0.02in}
  \paragraph{Importance of suboptimal responses in MPO}

Traditional preference optimization methods often underutilize the rich feedback signal available by relying solely on a single best-worst pair per query. However, models can produce diverse suboptimal responses, and learning from this can improve alignment. In Figure~\ref{fig:optimal-response-analysis}, we present results for \texttt{Llama-8B-Instruct} in the on-policy setting on AlpacaEval 2. We observe consistent gains as the number of responses per query increases from $n{=}2$ to $n{=}8$ under $\mpo$, highlighting the benefits of leveraging multiple responses during training.
% For instance, there is more than one way to get an answer right and wrong, and showing the model these ways, via the  \mpo{} loss function, improves downstream performance.  

% \vspace{0.04in}

% This underscores that incorporating additional suboptimal examples provides MPO with more comprehensive reward information, leading to more effective policy alignment than relying on a narrowly defined optimal pair.
  \end{minipage}\hfill          % \hfill inserts a flexible gutter
  \begin{minipage}[t]{0.4\textwidth}
    \vspace{-0.1in}
    \centering
    \includegraphics[width=\linewidth]{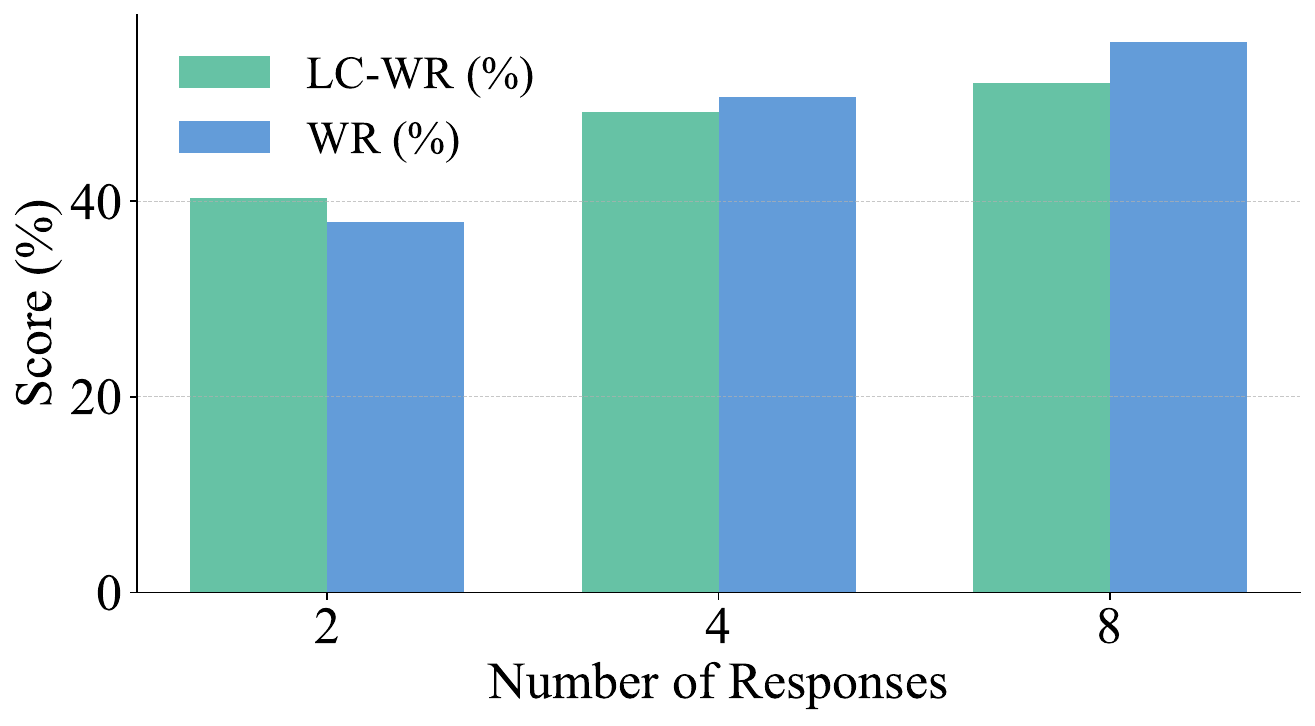}
    \vspace{-0.25in}
    \captionof{figure}{Downstream performance with number of model responses on Alpaca Eval2.}
    \label{fig:optimal-response-analysis}
  \end{minipage}
  % \vspace{-0.3in}
%   \begin{takeawaybox}
% \textbf{Key Takeaway:} Including additional sub-optimal responses in $\mpo$ improves policy learning, challenging the convention that only top-ranked answers drive alignment performance.
% \end{takeawaybox}
\end{figure}
\begin{takeawaybox}
\textbf{Key Takeaway:} Including additional sub-optimal responses in $\mpo$ improves policy learning, challenging the convention that only top-ranked answers drive alignment performance.
\end{takeawaybox}

\vspace{0.1in}
\begin{wraptable}{r}{0.5\textwidth}
\vspace{-0.15in}
\centering
\resizebox{0.48\textwidth}{!}{%
\begin{tabular}{lcccc}
\toprule
\multirow{2}{*}{\textbf{Method}} 
& \multicolumn{2}{c}{\textbf{Mistral (7B)}} 
& \multicolumn{2}{c}{\textbf{Llama-3 (8B)}} \\
\cmidrule(lr){2-3} \cmidrule(lr){4-5}
& \textbf{LC (\%)} & \textbf{WR (\%)} & \textbf{LC (\%)} & \textbf{WR (\%)} \\
\midrule
$\mpo$ & 28.2 & 29.4 & 49.0 & 50.6 \\
W-$\mpo$ & \textbf{29.5} & \textbf{30.6} & \textbf{52.1} & \textbf{52.5} \\
\bottomrule
\end{tabular}
}
\caption{Comparison of $\mpo$ and W-$\mpo$ on AlpacaEval2 for Mistral and Llama Instruct models.\vspace{-0.1in}}
\label{tab:w-mpo_and_mpo comparison}
\vspace{-1em}
\end{wraptable}

% \vspace{0.1in}
\paragraph{Impact of deviation-based weighting in W-MPO} 
We evaluate the impact of W-MPO's core deviation-based weighting by comparing it to MPO. Table \ref{tab:w-mpo_and_mpo comparison} underscores that deviation-aware weighting helps the policy learn from more separable supervision signals, leading to improved alignment.

\vspace{0.1in}
\begin{takeawaybox}
\textbf{Key Takeaway:} $\mpo$ supports the hypothesis that deviation based contrast amongst responses -- not just their absolute reward -- plays a critical role in preference optimization.
\end{takeawaybox}

\clearpage
% \begin{table}[!ht]
% \centering
% \resizebox{\textwidth}{!}{
% \setlength{\tabcolsep}{3pt}
% \begin{tabular}{l|c|cc|c}
% \hline
% \multirow{2}{*}{\textbf{Evaluator LM}} & \multirow{2}{*}{\textbf{Method}} & \multicolumn{2}{c|}{\textbf{AlpacaEval2}} & \multicolumn{1}{c}{\textbf{AlpacaEval1}} \\
%  &  & WR & LC-WR & WR \\
% \hline
% mistral-7b-instruct-v0.1 & $\mpo$ (k1 vs k2) (p = 0) & 11.14 & 16.06 & 88.81 \\
% mistral-7b-instruct-v0.1 & $\mpo$ (k1 vs k2) (p = 1) & 11.7 & 16.3 & \cellcolor{gray!25}88.93 \\
% mistral-7b-instruct-v0.1 & $\mpo$ (k1 vs k2) (p = 2) & \cellcolor{gray!25} 11.94 & \cellcolor{gray!25} 16.64 & 87.9 \\
% \hline
% \end{tabular}
% }
% \label{tab: weighing logit}
% \caption{Performance Comparison of Logit Weighting Methods in Group Contrastive Loss}
% \end{table}

% \begin{wrapfigure}{r}{0.5\columnwidth}
%     \centering
%     \vspace{-10pt}  % Adjust vertical space before the figure
%     \includegraphics[width=\linewidth]{images/aligned_compact_mpo_dpo_adjusted_labels (1).pdf}
%     \vspace{-12pt}  % Adjust vertical space after the figure
%     \caption{Comparison of $\mpo$ and $\dpo$ across varying training data sizes on AlpacaEval2.}
%     \label{fig:Training-Data-Analysis}
% \end{wrapfigure}

\begin{figure}[!h]               % one float, two columns inside
  \centering
  \begin{minipage}[t]{0.48\textwidth}
  \vspace{0.02in}
  \paragraph{$\mpo$ achieves superior gains even with limited data}
\vspace{0.05in}

We examine how $\mpo$ performs under data-constrained settings for Mistral (7B), compared to $\dpo$ across varying proportions of training data. As shown in Figure~\ref{fig:Training-Data-Analysis}, $\mpo$ outperforms $\dpo$ at varying percentages of training data. For instance, on LC-WR (\%), MPO at 25\% training data outperforms DPO at 50\%.
These improvements are a direct result of $\mpo$’s ability to leverage multiple responses per query, extracting more supervision signal.
  \end{minipage}\hfill          % \hfill inserts a flexible gutter
  \begin{minipage}[t]{0.48\textwidth}
    \vspace{0.00in}
    \centering
    \includegraphics[width=\linewidth]{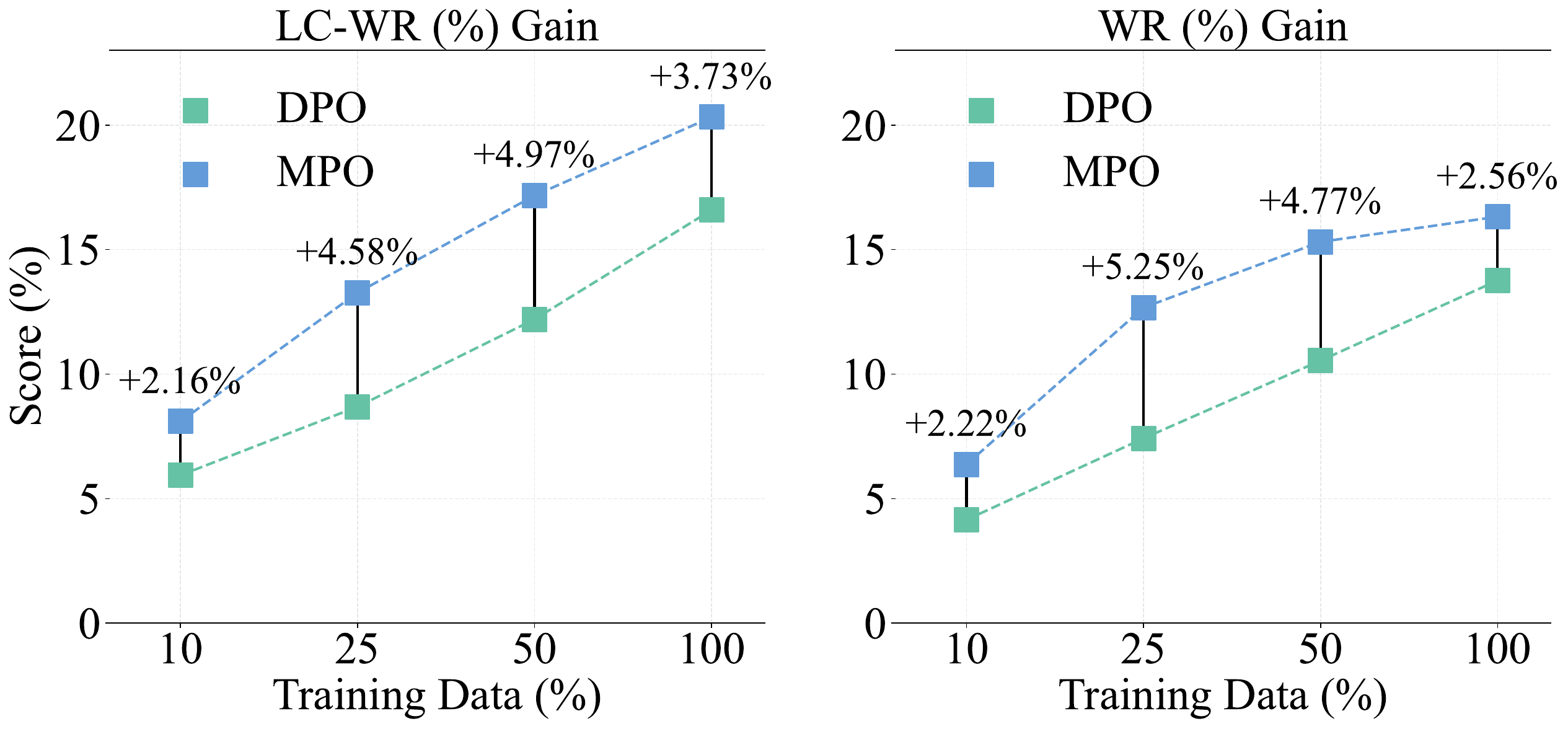}
    \captionof{figure}{Comparison of $\mpo$ and $\dpo$ across varying training data sizes on AlpacaEval2.}
    \label{fig:Training-Data-Analysis}
  \end{minipage}
    \vspace{0.05in}
  \begin{takeawaybox}
\textbf{Key Takeaway:} Even with limited training data, $\mpo$ consistently outperforms $\dpo$ by effectively utilizing multiple responses per query -- highlighting the importance of multiple responses per query, over multiple queries, at a constant (query, response) budget.
\end{takeawaybox}
  \vspace{0.05in}
\end{figure}
% \vspace{0.05in}
% \begin{wrapfigure}{r}{0.5\columnwidth}
%     \vspace{-10pt}  % Reduce space above
%     \includegraphics[width=1.0\linewidth]{images/aligned_compact_mpo_dpo_adjusted_labels (1).pdf}
%     \vspace{-15pt}  % Reduce space below
%     % \captionsetup{font=small}
%     \caption{Comparison of $\mpo$ and $\dpo$ across varying training data sizes on AlpacaEval2.\vspace{-25pt}}
%     \label{fig:Training-Data-Analysis}
% \end{wrapfigure}

% \paragraph{$\mpo$ Achieves Superior Gains Even with Limited Data}

% We examine how $\mpo$ performs under data-constrained settings, comparing its effectiveness with $\dpo$ across varying proportions of training data. As shown in Figure~\ref{fig:Training-Data-Analysis}, $\mpo$ outperforms $\dpo$ at all data scales.

% These improvements are a direct result of $\mpo$’s ability to leverage multiple responses per query, extracting more supervision signal even when the overall data size is small. 

% \vspace{-12pt}
% \begin{tcolorbox}[colback=gray!5!white, colframe=gray!50!black, boxrule=0.3pt, sharp corners=south]
% \textbf{Key Takeaway:} \textit{Even with limited training data, $\mpo$ consistently outperforms $\dpo$ by effectively utilizing multiple responses per query—highlighting the importance of response diversity over data quantity.}
% \end{tcolorbox}

\begin{figure}[!h]               % one float, two columns inside
  \centering
  \begin{minipage}[t]{0.48\textwidth}
  \vspace{0.00in}
    \centering
    \includegraphics[width=\linewidth]{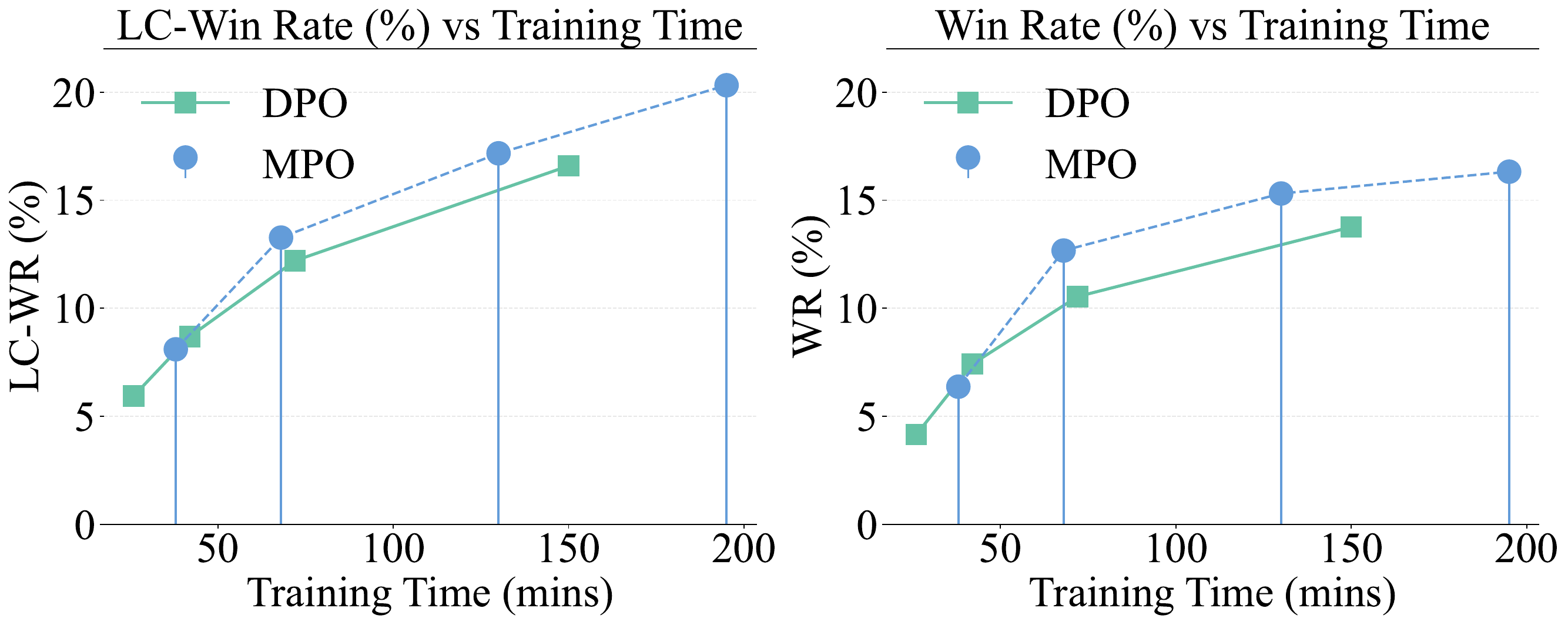}
    \captionof{figure}{Comparison of $\mpo$ and $\dpo$ on Alpaca\-Eval2
      under fixed training-time budgets.}
    \label{fig:time-efficiency}
  \end{minipage}\hfill          % \hfill inserts a flexible gutter
  \begin{minipage}[t]{0.48\textwidth}
    \vspace{0.02in}
    \paragraph{Efficiency of \textsc{MPO} under fixed training budgets}
    While $\mpo$ naturally incurs longer per-step training time because it
    processes multiple responses per query, we investigate whether this
    cost translates into superior performance within a fixed wall-clock
    budget.  Figure~\ref{fig:time-efficiency} shows that, despite the
    higher computation per batch, $\mpo$ consistently outperforms $\dpo$
    at every time checkpoint, achieving substantial gains in both LC-WR
    and overall WR for Mistral (7B).
  \end{minipage}
  \vspace{0.05in}
  \begin{takeawaybox}
\textbf{Key Takeaway:} Even under fixed training time, $\mpo$ outperforms $\dpo$, showing its greater sample efficiency and stronger use of richer preference signals.
\end{takeawaybox}
% \vspace{0.05in}
\end{figure}

% \begin{wrapfigure}{r}{0.5\columnwidth}
%     \centering
%     \vspace{-20pt}  % Adjust vertical space before the figure
%     \includegraphics[width=\linewidth]{images/training_time_vs_performance (2).pdf}
%     \vspace{-20pt}  % Adjust vertical space after the figure
%     \caption{Comparison of $\mpo$ and $\dpo$ on AlpacaEval2 under fixed training time budgets. \vspace{-10pt}}
%     \label{fig:time-efficiency}
% \end{wrapfigure}

% \paragraph{Efficiency of MPO Under Fixed Training Budgets:}

% While $\mpo$ naturally incurs longer per-step training time due to processing multiple responses per query, we investigate whether this added cost translates into superior performance within a fixed training budget. In Figure~\ref{fig:time-efficiency}, we compare $\mpo$ and $\dpo$ across varying training durations.

% Despite the higher computation per batch, $\mpo$ consistently outperforms $\dpo$ at every time checkpoint, achieving substantial gains in both LC-WR and WR.

% \begin{tcolorbox}[colback=gray!5!white, colframe=gray!50!black, boxrule=0.3pt, sharp corners=south]
% \textbf{Key Takeaway:} \textit{Even under fixed training time, $\mpo$ outperforms $\dpo$, showing its greater sample efficiency and stronger use of richer preference signals.}
% \end{tcolorbox}

% \begin{takeawaybox}
% \textbf{Key Takeaway:} Even under fixed training time, $\mpo$ outperforms $\dpo$, showing its greater sample efficiency and stronger use of richer preference signals.
% \end{takeawaybox}

\begin{figure}[!thbp]
    \centering
\includegraphics[width=1.0\columnwidth]{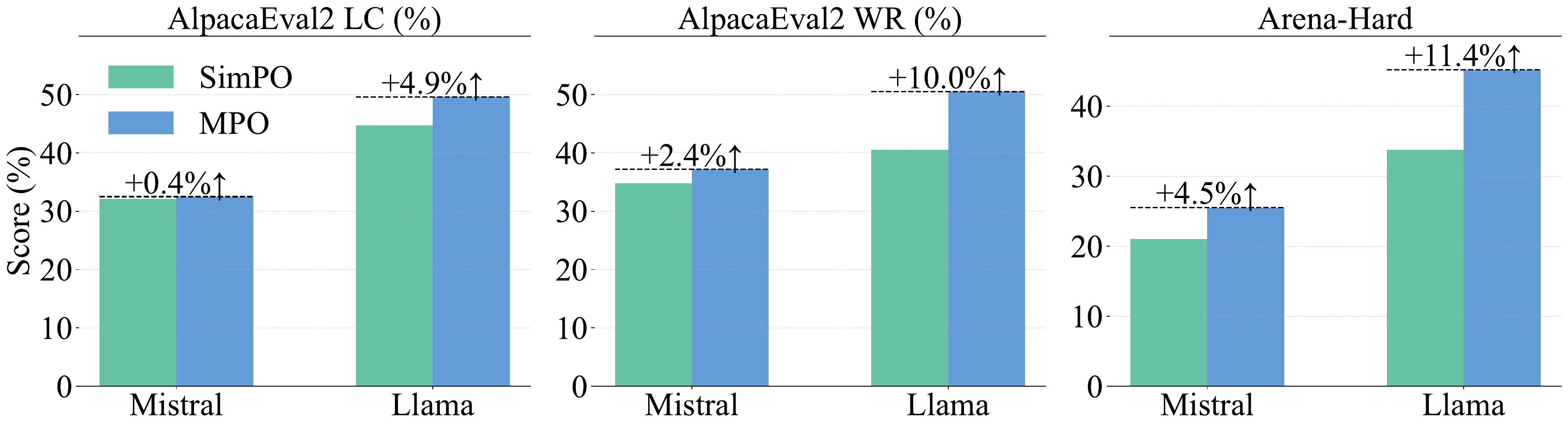}
    \vspace{-0.2in}
    \caption{Evaluating $\mpo$ (Reference-Free) vs $\simpo$ on Mistral and Llama instruct models.\vspace{-0.2in}}
    \label{fig:reference-free-analysis}
\end{figure}

% \vspace{12in}

\paragraph{Unifying $\dpo$ variants under the $\mpo$ setting} To demonstrate the generality, we explore how $\mpo$ can be instantiated on top of existing $\dpo$-based methods, such as KTO \citep{ethayarajh2024kto}, R-DPO \citep{park2024disentangling}, and SimPO \citep{meng2024simpo}. These methods share the same pairwise contrastive structure as $\dpo$, differing primarily in how they weigh or define the reward margin. Since $\mpo$ generalizes $\dpo$ to groupwise preference comparisons, it is naturally compatible with any loss that depends on contrast over preference data.

In Figure~\ref{fig:reference-free-analysis}, we highlight a concrete instantiation: $\mpo$ (Reference-Free), where we extend SimPO’s reward-based contrastive loss to operate in the multi-preference regime. This variant benefits from MPO’s ability to better utilize responses. We observe consistent improvements across evaluation on downstream datasets. The results indicate that $\mpo$ (Reference-Free) leverages responses more effectively than SimPO.
  \vspace{0.05in}
% \begin{tcolorbox}[colback=gray!5!white, colframe=gray!50!black, boxrule=0.3pt, sharp corners=south]
% \textbf{Key Takeaway:} \textit{$\mpo$ provides a principled generalization of $\dpo$-style contrastive objectives, allowing broader preference integration without altering the base reward design.}
% \end{tcolorbox}

\begin{takeawaybox}
\textbf{Key Takeaway:} $\mpo$ provides a principled generalization of $\dpo$-style contrastive objectives, allowing broader preference integration without altering the base reward design.
\end{takeawaybox}

% \begin{tcolorbox}[
%     title=Key Takeaway,
%     colback=blue!3!white,
%     colframe=blue!30!black,
%     boxrule=0.3pt,
%     arc=1mm,
%     sharp corners=south,
%     left=6pt,
%     right=6pt,
%     top=4pt,
%     bottom=4pt
% ]
% \textit{$\mpo$ provides a principled generalization of $\dpo$-style contrastive objectives, allowing broader preference integration without altering the base reward design.}
% \end{tcolorbox}
\bibliography{References}
\bibliographystyle{apalike}

% %%%%%%%%%%%%%%%%%%%%%%%%%%%%%%%%%%%%%%%%%%%%%%%%%%%%%%%%%%%%
\clearpage
\appendix
\vspace{1cm}
\hrule
\par\vspace{0.5cm}
{\Large\bfseries\centering \textsc 
{Supplementary Materials}
\par\vspace{0.5cm}}
\hrule
\vspace{0.5cm}
\noindent These supplementary materials provide additional details, derivations, and experimental results for our paper. The appendix is organized as follows:
\begin{itemize}[leftmargin=1em]
    \item Section \ref{sec:related_work_extended} presents a detailed related works section.

    \item Section \ref{sec:limitations} presents the core limitations of this work, and some avenues for future work.

    \item Section \ref{app:weighted_mpo_details} presents a detailed overview of a variant of our main method which is weighted MPO.
    
    \item Section \ref{sec:bias_analysis} presents a detailed bias analysis, demonstrating how incorporating multiple preferences reduces alignment bias.
    
    \item Section \ref{sec:gcl_vs_infonce_analysis} provides a comprehensive comparison between the Group Contrastive Loss and InfoNCA Loss, including detailed gradient analyses.
    
    \item Section \ref{sec:stationary_point_characterization} offers a thorough characterization of stationary points for both the InfoNCA and Weighted Contrastive Loss functions.
    
    \item Section \ref{sec:baselines} describes the baselines used for comparison in our experimental evaluations, including various DPO implementations and alternative approaches.

    \item Section \ref{app:experimental_details} describes the models, dataset and compute details for our experiments.
    
    \item Section \ref{sec:reward_loss_computation} provides the implementation details of the reward loss computation, including the actual code used in our experiments.

    \item Section \ref{sec:examples} provides examples of model responses using our model as compared with the base model.
\end{itemize}
\vspace{0.5cm}

\section{Related Work}
\label{sec:related_work_extended}
We will start this literature survey with a high level overview of the RLHF literature and then going deeper into the area of preference, and then multi-preference optimization relevant to our work.

\textbf{Broader RLHF Literature:}
Reinforcement Learning through Human feedback (RLHF) has emerged as a robust alignment algorithm for language models. The area broadly started of with works like Trust Region Policy Optimization (TRPO), and Proximal Policy Optimization (PPO) \citep{schulman2015trust,schulman2017proximal} which extend direct RL based methods by constraining the update space to within a trusted region and clipping policy updates to prevent instability respectively. Building upon earlier policy gradient methods \citep{sutton1999policy}, PPO has been successfully applied to alignment tasks in Reinforcement Learning from Human Feedback (RLHF), allowing language models to produce outputs aligned with human preferences \citep{ziegler2019fine, ouyang2022training}. Its simplicity and efficiency make it a standard approach for fine-tuning large-scale models. Prior to PPO, Trust Region Policy Optimization (TRPO) \citep{schulman2015trust} introduced constraints to improve learning stability, influencing the development of PPO. Early applications of policy gradient methods in natural language processing \citep{ranzato2015sequence} demonstrated the potential of reinforcement learning for language model training. 

\textbf{Preference Optimization:}
Direct Preference Optimization (DPO) simplifies the alignment of language models by optimizing a contrastive loss directly over paired preference data, bypassing the intermediate step of reward modeling \citep{rafailov2024direct}. Unlike RLHF, DPO does not require explicit reward functions, making it computationally efficient and suitable for limited preference datasets. Recent extensions of DPO, such as Identity Preference Optimization (IPO) \citep{azar2024general}, self-play preference optimization \citep{wu2024self}, preference ranking optimization \citep{song2024preference},  rejection sampling optimization \citep{liu2023statistical}, and generalized preference optimization \citep{tang2024generalized} are amongst the other recent works improve on the DPO method. 

Beyond the foundational pairwise approaches and their direct extensions, numerous recent works have proposed methods that adapt or refine DPO-like strategies, often eliminating the need for separate reward modeling or reference models.

\paragraph{Alternative Approaches Without Full Reward Modeling.}
\citet{ethayarajh2024kto} propose \textbf{KTO}, a framework inspired by prospect theory that directly learns whether a response is globally desirable or undesirable, thereby removing the requirement of having multiple positive examples per instruction. \citet{zeng2024token} focus on token-level alignment in \textbf{TDPO}, imposing forward KL divergence constraints for each token rather than solely for the final output. This fine-grained approach can mitigate the mode-collapse issues sometimes observed in sequence-level alignment. Meanwhile, \citet{dong2023raft} introduce a list-wise method called \textbf{RAFT}, where the model finetunes on the best response from each sampled set of \(k\) candidates, iteratively converging toward an optimal subset policy.  
By contrast, \citet{yuan2023rrhf} center on rank-based supervision through \textbf{RRHF}, which combines a rank loss with standard supervised signals to ensure the model maintains stronger probabilities on higher-ranked (i.e., better) responses and less on suboptimal responses.

\paragraph{Enhancing DPO with Additional Objectives and Training Schemes.} 
Other works further modify or reinterpret the DPO loss to incorporate new constraints or to remove the need for a reference model. \citet{chen2024self} propose \textbf{SPIN}, which treats the model as part of a two-player adversarial game, obviating separate reward modeling by training with a discriminator that distinguishes human from machine responses. \textbf{CPO} \citep{xu2024contrastive} reworks the DPO objective by removing the reference-model term and adding a behavior cloning regularizer. Similarly, \textbf{ORPO} \citep{hong2024orpo} folds preference optimization into a negative log-likelihood objective via an odds-ratio penalty, thereby unifying supervised fine-tuning (SFT) and preference training.  
In \textbf{SimPO}, \citet{meng2024simpo} remove the reference model and incorporate a length normalization to address verbosity issues that can skew preference data. Likewise, \textbf{R-DPO} \citep{park2024disentangling} and \textbf{LD-DPO} \citep{liu2024iterative} specifically tackle length bias by injecting additional regularizers or by explicitly separating length-based preferences from other factors. For instance, LD-DPO modifies the training set to handle length constraints, preventing performance drops on standard benchmarks while mitigating length exploitation in preference tasks.

\paragraph{Refining Training Regimens for Preference Data.}
A final family of works emphasizes how training procedures or data usage can be systematically improved. For instance, \citet{kim2024sdpo} propose \textbf{sDPO}, a step-wise learning method partitioning preference data to stabilize training. \textbf{IRPO} \citep{pang2024iterative} enhances chain-of-thought reasoning by incorporating a negative log-likelihood term for the chosen solution path, thus nudging LLMs toward robust multi-step reasoning. \textbf{OFS-DPO} \citep{qi2024online} trains two LoRA modules at different paces—one faster, one slower—to sustain gradient momentum and to adapt more efficiently. Lastly, \citet{yuan2024following} tackle verbosity with \textbf{LIFT-DPO}, an approach that augments preference data with length-control instructions, ensuring that the model does not exploit response length to inflate its preference scores.

\textbf{Multi-Preference Optimization:}
Traditional preference optimization methods, like DPO, consider pairwise comparisons. However, datasets such as UltraFeedback \citep{cui2023ultrafeedback} highlight the necessity of multi-preference optimization. Multi-preference methods, such as InfoNCA \citep{chen2024noise}, leverage all available positive and negative responses simultaneously, reducing alignment bias and better approximating the true preference distribution. These methods mitigate limitations inherent to pairwise approaches by incorporating the diversity of acceptable and suboptimal responses. Earlier works in search have also used multiple user preferences to optimize models in various applications such as search \citep{joachims2002optimizing}.

\textbf{Reward Modeling in Preferences:}
Reward modeling is essential for translating qualitative human feedback into quantitative metrics that guide AI behavior optimization. Traditional methods, such as Reinforcement Learning from Human Feedback (RLHF), utilize reward models trained on human annotations to inform policy updates \citep{christiano2017deep, stiennon2020learning}. Early approaches like inverse reinforcement learning \citep{ng2000algorithms} and apprenticeship learning \citep{abbeel2004apprenticeship} demonstrated the feasibility of inferring reward functions from observed behaviors. Recent advancements have diversified reward modeling techniques. For instance, the Adversarial Preference Optimization (APO) framework employs adversarial training to adapt reward models to the evolving generation distribution of language models \citep{cheng2023adversarial}.

\textbf{Noise Contrastive Estimation and InfoNCA:}
Contrastive learning, particularly methods like InfoNCE \citep{oord2018representation}, maximizes mutual information between positive samples while discriminating against negatives. In the language domain, \citet{klein2024contrastive} leverage a perplexity-based contrastive objective to reduce toxic language generation while preserving the model’s overall utility.
InfoNCA adapts these principles for preference optimization, aligning responses with scalar rewards through noise-contrastive estimation \citep{chen2024noise}. Despite its strengths, InfoNCA can overemphasize less informative negative samples, which motivates methods like $\swepo$ that dynamically weigh responses based on deviation from the mean reward. 

\textbf{UltraFeedback Dataset:}
The UltraFeedback dataset \citep{cui2023ultrafeedback} is a significant advancement in preference-based training resources. It comprises GPT-4 annotated feedback for over 64,000 instructions, including scalar reward evaluations. UltraFeedback has been pivotal in developing models like UltraLM-13B-PPO and UltraRM, which achieve state-of-the-art performance across benchmarks such as AlpacaEval. This dataset's granularity enables advanced preference optimization methods like $\swepo$ to leverage diverse response quality levels effectively.

\section{Limitations}
\label{sec:limitations}

\noindent While \mpo\ and W-\mpo\ demonstrate strong empirical performance and offer theoretical benefits, we identify several areas for further exploration and acknowledge certain design considerations:

\subsection{Reward Signal Quality and Partitioning}
\label{ssec:limitation_reward_quality}
The partitioning into preferred ($\mathcal{Y}^+$) and dispreferred ($\mathcal{Y}^-$) sets, along with the deviation calculation in W-\mpo, utilizes input scalar reward scores $S(y)$. While W-\mpo's absolute deviation weighting and the underlying contrastive nature of \mpo\ are designed for greater robustness compared to methods directly regressing on or using signed advantages from potentially noisy reward models, the initial quality of $S(y)$ still influences these preliminary steps. Significant miscalibration or systemic biases in the reward source could affect the optimality of the partitions. Future research could explore adaptive partitioning mechanisms or methods to integrate reward uncertainty more directly.

\subsection{Computational Scaling with Extremely Large Response Sets}
\label{ssec:limitation_computational_scaling}
\mpo\ and W-\mpo\ process all $N_x$ responses available for a given query to compute the loss. This comprehensive approach enables richer supervision than pairwise DPO. However, for scenarios generating an extremely large number of candidate responses per query (e.g., hundreds), the per-step computational demand might increase. Investigating efficient subset selection techniques that identify the most informative responses from very large pools, while retaining the benefits of the MPO framework, presents a valuable avenue for future work to enhance scalability further.

\subsection{Exploring Richer Preference Structures Beyond Binary Contrast}
\label{ssec:limitation_beyond_binary}
The current \mpo\ framework establishes a robust binary contrast between a single "accepted" set $\mathcal{Y}^+$ and a single "rejected" set $\mathcal{Y}^-$. While this generalization of DPO is effective, human preferences can often be more nuanced. Future work could investigate extending \mpo\ to accommodate multiple ordered preference sets (e.g., $\mathcal{Y}^{++} > \mathcal{Y}^+ > \mathcal{Y}^- > \mathcal{Y}^{--}$) or a finer spectrum of preference levels. This might involve further generalizations of the underlying groupwise Bradley-Terry model to capture such richer, multi-level preference structures.

\noindent Despite these considerations, \mpo\ and W-\mpo\ offer a robust and effective advancement for leveraging multi-preference data in LLM alignment.
\section{Detailed Formulation of Weighted \mpo{} (W-\mpo)}
\label{app:weighted_mpo_details}

While the core Multi-Preference Optimization (\mpo) framework, as presented in Section \ref{sec:mpo_methodology}, effectively leverages sets of preferred ($\mathcal{Y}_x^+$) and dispreferred ($\mathcal{Y}_x^-$) responses, it treats all responses within each set with uniform importance. However, many modern alignment pipelines, particularly those employing on-policy or iterative generation and evaluation, utilize powerful reward models (RMs) that provide fine-grained scalar quality scores $S(y)$ for each generated response $y$. These scores often reveal significant variance in quality even within the pre-defined preferred or dispreferred sets. Weighted \mpo{} (W-\mpo) is designed to exploit this richer, quantitative feedback.

\subsection{Motivation for Weighted \mpo}
\label{app_ssec:motivation_wmpo}

In contemporary LLM alignment, especially in on-policy settings (e.g., RLHF with a learned reward model, or iterative DPO-like methods where a model generates responses that are then scored), obtaining exhaustive human pairwise preferences for all generated outputs can be prohibitively expensive or slow. Instead, an automated reward model $RM(y \mid x)$ is often used to assign a scalar quality score $S(y)$ to each response $y$ generated for a prompt $x$. These scalar scores $S(y)$ offer several advantages:
\begin{enumerate}
    \item \textbf{Granularity:} They provide a more nuanced measure of quality than simple binary (chosen/rejected) or pairwise preferences. A response $y_1$ might be preferred over $y_2$, but $S(y_1)$ can also indicate \textit{how much} better it is, or how it compares to a global quality scale.
    \item \textbf{Efficiency:} RMs can score a large volume of responses quickly, enabling larger-scale and more frequent preference data collection, crucial for iterative alignment.
    \item \textbf{Handling Multiple Responses:} When multiple ($N_x > 2$) responses are generated per prompt (e.g., via diverse beam search, multiple sampling seeds, or from different model variants), RMs can score all of them, providing a rich landscape of quality.
\end{enumerate}

The core \mpo{} method, by operating on sets $\mathcal{Y}_x^+$ and $\mathcal{Y}_x^-$ (which might themselves be derived from $S(y)$ via a threshold like $S_{\text{mean}}(x)$), does not fully capitalize on the information that some responses in $\mathcal{Y}_x^+$ are "more preferred" (higher $S(y)$) than others, or that some in $\mathcal{Y}_x^-$ are "more dispreferred" (lower $S(y)$). Responses that are significant outliers in terms of their quality scores (i.e., those deviating most from an average or baseline quality) are often the most informative for training. For instance, an exceptionally good response provides a strong positive signal, while a particularly problematic response offers a clear negative signal to learn from.

W-\mpo{} aims to incorporate this finer-grained information by weighting the contribution of each response $y$ in the \mpo{} loss based on its scalar score $S(y)$. Specifically, it uses the deviation of $S(y)$ from a reference point (e.g., the mean score $S_{\text{mean}}(x)$ for the current prompt) to modulate the response's impact. This allows W-\mpo{} to:
\begin{itemize}
    \item \textbf{Prioritize Informative Samples:} Give more importance to responses that are either exceptionally good or exceptionally bad, as these are strong learning signals.
    \item \textbf{Implement a Data-Driven Curriculum:} Naturally focus the model's attention on responses from which it can learn the most about the boundaries of desired behavior.
    \item \textbf{Enhance Stability and Robustness:} By considering the full spectrum of quality scores, W-\mpo{} can potentially lead to more stable and robust alignment compared to methods that only consider coarse preference signals.
\end{itemize}
This approach is particularly relevant when $S(y)$ comes directly from a reward model, as it allows the alignment process to more faithfully reflect the nuances captured by that RM.

\subsection{W-\mpo{} Formulation}
\label{app_ssec:wmpo_formulation}

Let $\mathcal{D}_{\text{W-MPO}}$ be a dataset where each instance consists of a query $x$, a set of $N_x$ responses $\mathcal{Y}_x = \{y_1, \dots, y_{N_x}\}$ for that query, and their corresponding scalar quality scores $\{S(y_1), \dots, S(y_{N_x})\}$.

\paragraph{Response Partitioning and Deviation Calculation.}
First, for each query $x$, we calculate the mean quality score:
\begin{equation}
S_{\text{mean}}(x) = \frac{1}{N_x} \sum_{j=1}^{N_x} S(y_j).
\label{app_eq:mean_score_calc_wmpo}
\end{equation}
The responses are then partitioned into preferred and dispreferred sets:
\begin{align}
\mathcal{Y}_x^+ &= \{ y \in \mathcal{Y}_x \mid S(y) > S_{\text{mean}}(x) \}, \label{app_eq:Y_plus_from_scores_wmpo} \\
\mathcal{Y}_x^- &= \{ y \in \mathcal{Y}_x \mid S(y) \leq S_{\text{mean}}(x) \}. \label{app_eq:Y_minus_from_scores_wmpo}
\end{align}
If $\mathcal{Y}_x^+$ is empty, the sample for query $x$ is typically skipped.

Next, we define a signed \textit{deviation term} $\Delta W(y)$ for each response $y \in \mathcal{Y}_x$:
\begin{equation}
\Delta W(y) = S(y) - S_{\text{mean}}(x).
\label{app_eq:delta_W}
\end{equation}
This term, $\Delta W(y)$, is positive for $y \in \mathcal{Y}_x^+$ and non-positive for $y \in \mathcal{Y}_x^-$. It quantifies how much better or worse a response is compared to the average quality for that query.

\paragraph{Additive Logit Adjustment.}
W-\mpo{} incorporates this deviation $\Delta W(y)$ by additively adjusting the logit (scaled implicit preference score) of each response. The base implicit score is $r_\theta(y \mid x) = \log (\pi_\theta(y \mid x) / \pi_{\text{ref}}(y \mid x))$, as defined in Eq. \ref{eq:implicit_score_prelim}. The modified logit, $r'_{\theta}(y \mid x)$, used in W-\mpo{} is:
\begin{equation}
r'_{\theta}(y \mid x; \alpha_w) = \beta \cdot r_\theta(y \mid x) + \alpha_w \cdot \Delta W(y),
\label{app_eq:wmpo_additive_logit_single}
\end{equation}
where $\beta$ is the inverse temperature scaling the base DPO-like score, and $\alpha_w$ is a hyperparameter controlling the magnitude of the deviation-based adjustment. A positive $\alpha_w$ will increase the effective logit for responses with $S(y) > S_{\text{mean}}(x)$ and decrease it for responses with $S(y) < S_{\text{mean}}(x)$, with the magnitude of adjustment proportional to $|\Delta W(y)|$.

\paragraph{W-\mpo{} Loss Function.}
The W-\mpo{} loss function retains the structure of the core \mpo{} loss (Eq. \ref{eq:mpo_loss_definition}) but utilizes these modified logits $r'_{\theta}(y \mid x; \alpha_w)$:
\begin{equation}
\mathcal{L}_{\text{W-\mpo}}(\theta) = -\mathbb{E}_{(x, \mathcal{Y}_x, \{S(y)\}) \sim \mathcal{D}_{\text{W-MPO}}} \left[ \log \left( \frac{\sum_{y \in \mathcal{Y}_x^+} \exp(r'_{\theta}(y \mid x; \alpha_w))}{\sum_{y' \in \mathcal{Y}_x^+ \cup \mathcal{Y}_x^-} \exp(r'_{\theta}(y' \mid x; \alpha_w))} \right) \right].
\label{app_eq:wmpo_loss_single}
\end{equation}
This additive adjustment in the logit space (Eq. \ref{app_eq:wmpo_additive_logit_single}) is equivalent to applying a multiplicative exponential weight $w_{\text{exp}}(y) = \exp(\alpha_w \cdot \Delta W(y))$ to the original base strength term $\exp(\beta \cdot r_\theta(y \mid x))$ within the sums of the \mpo{} loss. The W-\mpo{} objective thus encourages the policy $\pi_\theta$ to more strongly prefer/disprefer responses that are further above/below the mean quality score, respectively.

\subsection{Algorithm for W-\mpo}
\label{app_ssec:algorithm_wmpo}

The detailed training procedure for W-\mpo{} is presented in Algorithm \ref{alg:appendix_wmpo_detailed}.

\begin{algorithm}[!htbp]
   \caption{Detailed Weighted Multi-Preference Optimization (W-\mpo)}
   \label{alg:appendix_wmpo_detailed}
   \begin{algorithmic}[1]
       \STATE {\bfseries Input:} Initial policy parameters $\theta_0$; dataset $\mathcal{D}_{\text{W-MPO}}$, where each sample contains $(x, \mathcal{Y}_x, \{S(y)\}_{y \in \mathcal{Y}_x})$; reference model $\pi_{\text{ref}}$; inverse temperature $\beta$; W-\mpo{} weight scaling hyperparameter $\alpha_w$; learning rate $\eta$; number of training iterations/epochs $T$.
       \STATE {\bfseries Output:} Optimized model parameters $\theta_T$.
       \STATE Initialize policy model parameters $\theta \leftarrow \theta_0$.
       \FOR{$t = 1$ {\bfseries to} $T$}
           \STATE Initialize total loss for epoch/iteration $\mathcal{L}_{\text{epoch}} \leftarrow 0$.
           \FORALL{sample $(x, \mathcal{Y}_x, \{S(y)\}_{y \in \mathcal{Y}_x})$ in $\mathcal{D}_{\text{W-MPO}}$ (typically processed in mini-batches)}
               \STATE Compute $S_{\text{mean}}(x) \gets \frac{1}{|\mathcal{Y}_x|} \sum_{y_j \in \mathcal{Y}_x} S(y_j)$. (Ref. Eq. \ref{app_eq:mean_score_calc_wmpo})
               \STATE Define $\mathcal{Y}_x^+ \gets \{ y \in \mathcal{Y}_x \mid S(y) > S_{\text{mean}}(x) \}$. (Ref. Eq. \ref{app_eq:Y_plus_from_scores_wmpo})
               \STATE Define $\mathcal{Y}_x^- \gets \{ y \in \mathcal{Y}_x \mid S(y) \leq S_{\text{mean}}(x) \}$. (Ref. Eq. \ref{app_eq:Y_minus_from_scores_wmpo})
               \IF{$|\mathcal{Y}_x^+| = 0$} 
                   \STATE {\bfseries continue} \COMMENT{Skip sample if no preferred responses after partitioning}
               \ENDIF
               
               \STATE Let $\mathcal{Y}_{\text{all}} = \mathcal{Y}_x^+ \cup \mathcal{Y}_x^-$.
               \STATE Create a map $R'_{\theta}$ to store modified logits.
               \FOR{$y \in \mathcal{Y}_{\text{all}}$}
                   \STATE Compute base score: $r_\theta(y \mid x) \gets \log\left(\frac{\pi_\theta(y \mid x)}{\pi_{\text{ref}}(y \mid x)}\right)$. (Ref. Eq. \ref{eq:implicit_score_prelim})
                   \STATE Compute deviation term: $\Delta W(y) \gets S(y) - S_{\text{mean}}(x)$. (Ref. Eq. \ref{app_eq:delta_W})
                   \STATE Compute modified logit: $r'_{\theta}(y \mid x) \gets \beta \cdot r_\theta(y \mid x) + \alpha_w \cdot \Delta W(y)$. (Ref. Eq. \ref{app_eq:wmpo_additive_logit_single})
                   \STATE $R'_{\theta}[y] \gets r'_{\theta}(y \mid x)$.
               \ENDFOR
               
               \STATE Numerator term $N_x \gets \sum_{y \in \mathcal{Y}_x^+} \exp(R'_{\theta}[y])$.
               \STATE Denominator term $D_x \gets \sum_{y' \in \mathcal{Y}_{\text{all}}} \exp(R'_{\theta}[y'])$.
               
               \STATE Sample loss $L_x \gets -\log(N_x / D_x)$, ensuring $N_x/D_x \in (0, 1]$. (Ref. Eq. \ref{app_eq:wmpo_loss_single})
               \STATE Accumulate loss (e.g., $\mathcal{L}_{\text{epoch}} \leftarrow \mathcal{L}_{\text{epoch}} + L_x$).
           \ENDFOR
           \STATE Compute gradient $\nabla_\theta \mathcal{L}_{\text{epoch}}$ (typically averaged over batch size).
           \STATE Update model parameters: $\theta \leftarrow \theta - \eta \nabla_\theta \mathcal{L}_{\text{epoch}}$.
       \ENDFOR
       \RETURN $\theta$.
   \end{algorithmic}
\end{algorithm}
\section{Bias Analysis:}
\label{sec:bias_analysis}

In the first part of this section section, we analyze how the number of positive and negative examples per query, \( k \), affects the bias with respect to an attribute \( a(y) \). We provide a formal theorem establishing the relationship between bias and \( k \), followed by a corollary discussing the behavior as \( k \to \infty \).

The reason for this analysis is to show that multi-preference sampling of accepted and rejected answers from a distribution is better than using a single sample as DPO does. The more accepted and rejected samples you have, the lower the bias, provably.

\subsection{Assumptions}

We make the following assumptions:

\begin{enumerate}
    \item \textbf{Attribute Function}: Let \( a(y): \mathcal{Y}_x \rightarrow \mathbb{R} \) be an attribute function mapping responses to real numbers (e.g., response length).
    \item \textbf{Finite Variance}: The attribute \( a(y) \) has finite variance over the acceptable response distribution \( \mathcal{A}_x \) for each query \( x \), i.e., \( \operatorname{Var}_{y \sim \mathcal{A}_x}[a(y)] = \sigma_{\mathcal{A}_x}^2 < \infty \).
    \item \textbf{Independent Sampling}: Responses are independently sampled from their respective distributions.
    \item \textbf{Model Capacity}: The model can represent the true distribution given sufficient data.
    \item \textbf{Uniform Bounded Variance}: There exists a constant \( \sigma_{\max}^2 \) such that \( \sigma_{\mathcal{A}_x}^2 \leq \sigma_{\max}^2 \) for all \( x \in \mathcal{X} \).
\end{enumerate}

\subsection{Bias Definition}

The bias with respect to attribute \( a \) is defined as:

\begin{equation}
B^{(k)} = \left| \mu_\theta^{(k)} - \mu_{\mathcal{A}} \right|,
\end{equation}

where:

\begin{itemize}
    \item \( \mu_\theta^{(k)} \) is the expected attribute value under the model after training with \( k \) positive and \( k \) negative samples per query.
    \item \( \mu_{\mathcal{A}} = \mathbb{E}_{x \sim \mathcal{X}} \left[ \mu_{\mathcal{A}_x} \right] \), with \( \mu_{\mathcal{A}_x} = \mathbb{E}_{y \sim \mathcal{A}_x} \left[ a(y) \right] \).
\end{itemize}

\subsection{Main Bias Result}

\begin{theorem}
Under the stated assumptions, the expected bias \( \mathbb{E}[B^{(k)}] \) decreases with the number of samples \( k \) as:

\begin{equation}
\mathbb{E}[B^{(k)}] \leq \frac{C}{\sqrt{k}},
\end{equation}

where \( C = \sigma_{\max} \) is a constant depending on the maximum variance of \( a(y) \) over the acceptable responses.
\end{theorem}

\begin{proof}
For each query \( x \), consider the sample mean of the attribute over the \( k \) positive responses:

\begin{equation}
\overline{a}_x^{(k)} = \frac{1}{k} \sum_{i=1}^{k} a(y_{x,i}^+), \quad y_{x,i}^+ \sim \mathcal{A}_x.
\end{equation}

Since the \( y_{x,i}^+ \) are independent and identically distributed samples from \( \mathcal{A}_x \), the expected value and variance of \( \overline{a}_x^{(k)} \) are:

\begin{align}
\mathbb{E}\left[ \overline{a}_x^{(k)} \right] &= \mu_{\mathcal{A}_x}, \\
\operatorname{Var}\left( \overline{a}_x^{(k)} \right) &= \frac{\sigma_{\mathcal{A}_x}^2}{k} \leq \frac{\sigma_{\max}^2}{k}.
\end{align}

Using the fact that for any random variable \( Z \) with finite variance, the expected absolute deviation from its mean satisfies:

\begin{equation}
\mathbb{E}\left[ \left| Z - \mathbb{E}[Z] \right| \right] \leq \sqrt{ \operatorname{Var}[Z] },
\end{equation}

we have:

\begin{equation}
\mathbb{E}\left[ \left| \overline{a}_x^{(k)} - \mu_{\mathcal{A}_x} \right| \right] \leq \sqrt{ \frac{\sigma_{\mathcal{A}_x}^2}{k} } \leq \frac{\sigma_{\max}}{\sqrt{k}}.
\end{equation}

Averaging over all queries \( x \in \mathcal{X} \):

\begin{equation}
\mathbb{E}_{x} \left[ \left| \overline{a}_x^{(k)} - \mu_{\mathcal{A}_x} \right| \right] \leq \frac{\sigma_{\max}}{\sqrt{k}}.
\end{equation}

Since \( \mu_\theta^{(k)} = \mathbb{E}_{x} \left[ \overline{a}_x^{(k)} \right] \) and \( \mu_{\mathcal{A}} = \mathbb{E}_{x} \left[ \mu_{\mathcal{A}_x} \right] \), the expected bias is:

\begin{align}
\mathbb{E}[B^{(k)}] &= \left| \mu_\theta^{(k)} - \mu_{\mathcal{A}} \right| = \left| \mathbb{E}_{x} \left[ \overline{a}_x^{(k)} - \mu_{\mathcal{A}_x} \right] \right| \\
&\leq \mathbb{E}_{x} \left[ \left| \overline{a}_x^{(k)} - \mu_{\mathcal{A}_x} \right| \right] \leq \frac{\sigma_{\max}}{\sqrt{k}}.
\end{align}

Thus, the expected bias decreases with \( k \) as \( \frac{1}{\sqrt{k}} \).

\end{proof}

\begin{corollary}
As \( k \to \infty \), the expected bias \( \mathbb{E}[B^{(k)}] \) approaches zero:

\begin{equation}
\lim_{k \to \infty} \mathbb{E}[B^{(k)}] = 0.
\end{equation}
\end{corollary}

\paragraph{Implications}

This theorem establishes a quantitative relationship between the bias \( B^{(k)} \) and the number of samples \( k \). It shows that incorporating multiple positive and negative responses per query reduces the bias with respect to attribute \( a(y) \) at a rate proportional to \( 1/\sqrt{k} \). As \( k \) increases, the model's expected attribute value converges to the true expected attribute value over acceptable responses, leading to better alignment with human preferences.

\section{Differentiating the Group Contrastive Loss from InfoNCE Loss}
\label{sec:gcl_vs_infonce_analysis}

In this subsection, we compare our proposed weighted contrastive loss function with the InfoNCA loss function. We present both loss functions, derive their gradients rigorously, and characterize their stationary points. Based on this characterization, we discuss the properties of the convergence points in terms of what the models learn and their alignment with human preferences.

\subsection{Definitions of Loss Functions}

\paragraph{InfoNCA Loss Function}

The InfoNCA loss function is defined as:

\[
L_{\text{InfoNCA}} = - \sum_{i=1}^K p_i^{\text{target}} \log p_i^{\text{model}},
\]

where \( p_i^{\text{target}} \) represents the target probability for the \(i\)-th response, calculated as

\[
p_i^{\text{target}} = \frac{e^{r(x, y_i)/\alpha}}{\sum_{j=1}^K e^{r(x, y_j)/\alpha}},
\]

and \( p_i^{\text{model}} \) denotes the model's predicted probability for the \(i\)-th response, given by

\[
p_i^{\text{model}} = \frac{e^{s_\theta(y_i \mid x)}}{\sum_{j=1}^K e^{s_\theta(y_j \mid x)}}.
\]

In this context, \( x \) is the instruction or prompt provided to the model, and \( \{y_i\}_{i=1}^K \) represents a set of \( K \) responses generated for the instruction \( x \). The term \( r(x, y_i) \) is the reward associated with the response \( y_i \), while \( s_\theta(y_i \mid x) = \log\left(P_\theta(y_i \mid x)/P_{\text{ref}}(y_i \mid x) \right) \) is the score for response \( y_i \). The parameter \( \alpha \) serves as a temperature parameter that controls the influence of the reward, and \( K \) is the total number of responses considered for the instruction \( x \).

\paragraph{Weighted Contrastive Loss Function}

Our proposed weighted contrastive loss function is expressed as:

\[
L_{\text{weighted}} = -\log \left( \frac{\sum\limits_{i \in Y^+} w_i e^{s_\theta(y_i \mid x)}}{\sum\limits_{j=1}^K w_j e^{s_\theta(y_j \mid x)}} \right),
\]

where \( Y^+ \) is the set of positive responses with rewards above the mean, defined as \( Y^+ = \{ y_i \mid S_i > S_{\text{mean}} \} \). Each response \( y_i \) is assigned a weight \( w_i = e^{\alpha \delta_i} \), where \( \delta_i \) is the deviation of the reward score \( S_i \) from the mean reward score \( S_{\text{mean}} \). Specifically, \( \delta_i = S_i - S_{\text{mean}} \) for responses in \( Y^+ \) and \( \delta_i = S_{\text{mean}} - S_i \) for responses not in \( Y^+ \). The mean reward score \( S_{\text{mean}} \) is calculated as 

\[
S_{\text{mean}} = \frac{1}{K} \sum_{j=1}^K S_j,
\]

where \( K \) is the total number of responses for the query \( x \). The term \( s_\theta(y_i \mid x) \) denotes the model's logit for response \( y_i \), and \( \alpha \) is a scaling hyperparameter that controls the influence of the deviation \( \delta_i \).

\subsection{Gradient Analysis}

To understand how each loss function influences the model during training, we derive the gradients with respect to the model logits $s_\theta(y_i \mid x)$ for both methods.

\paragraph{Gradient of InfoNCA Loss}
\begin{lemma}
The gradient of the InfoNCA loss with respect to the model logits $s_\theta(y_i \mid x)$ is:

\begin{equation}
\frac{\partial L_{\text{InfoNCA}}}{\partial s_\theta(y_i \mid x)} = p_i^{\text{model}} - p_i^{\text{target}}.
\end{equation}
\end{lemma}

\begin{proof}
The InfoNCA loss is:

\begin{equation}
L_{\text{InfoNCA}} = - \sum_{k=1}^K p_k^{\text{target}} \log p_k^{\text{model}}.
\end{equation}

Our goal is to compute $\frac{\partial L_{\text{InfoNCA}}}{\partial s_\theta(y_i \mid x)}$.

Since $p_k^{\text{target}}$ does not depend on $s_\theta(y_i \mid x)$ (the rewards are constants with respect to the model parameters), the derivative only affects the terms involving $p_k^{\text{model}}$.

First, express $\log p_k^{\text{model}}$ explicitly:

\begin{equation}
\log p_k^{\text{model}} = s_\theta(y_k \mid x) - \log \left( \sum_{j=1}^K e^{s_\theta(y_j \mid x)} \right).
\end{equation}

Now, compute the derivative of $\log p_k^{\text{model}}$ with respect to $s_\theta(y_i \mid x)$:

\begin{equation}
\frac{\partial \log p_k^{\text{model}}}{\partial s_\theta(y_i \mid x)} = \frac{\partial s_\theta(y_k \mid x)}{\partial s_\theta(y_i \mid x)} - \frac{\partial}{\partial s_\theta(y_i \mid x)} \log \left( \sum_{j=1}^K e^{s_\theta(y_j \mid x)} \right).
\end{equation}

Compute each term separately.

First term:

\begin{equation}
\frac{\partial s_\theta(y_k \mid x)}{\partial s_\theta(y_i \mid x)} = \delta_{ik},
\end{equation}

where $\delta_{ik}$ is the Kronecker delta, equal to $1$ if $i = k$ and $0$ otherwise.

Second term:

Let $Z = \sum_{j=1}^K e^{s_\theta(y_j \mid x)}$. Then,

\begin{equation}
\frac{\partial}{\partial s_\theta(y_i \mid x)} \log Z = \frac{1}{Z} \frac{\partial Z}{\partial s_\theta(y_i \mid x)}.
\end{equation}

Compute $\frac{\partial Z}{\partial s_\theta(y_i \mid x)}$:

\begin{equation}
\frac{\partial Z}{\partial s_\theta(y_i \mid x)} = e^{s_\theta(y_i \mid x)}.
\end{equation}

Therefore,

\begin{equation}
\frac{\partial}{\partial s_\theta(y_i \mid x)} \log Z = \frac{e^{s_\theta(y_i \mid x)}}{Z} = p_i^{\text{model}}.
\end{equation}

Putting it all together:

\begin{align}
\frac{\partial \log p_k^{\text{model}}}{\partial s_\theta(y_i \mid x)} &= \delta_{ik} - p_i^{\text{model}}.
\end{align}

Now, compute the gradient of the loss:

\begin{align}
\frac{\partial L_{\text{InfoNCA}}}{\partial s_\theta(y_i \mid x)} &= - \sum_{k=1}^K p_k^{\text{target}} \frac{\partial \log p_k^{\text{model}}}{\partial s_\theta(y_i \mid x)} \\
&= - \sum_{k=1}^K p_k^{\text{target}} \left( \delta_{ik} - p_i^{\text{model}} \right) \\
&= - \left( p_i^{\text{target}} - p_i^{\text{model}} \sum_{k=1}^K p_k^{\text{target}} \right).
\end{align}

Since $\sum_{k=1}^K p_k^{\text{target}} = 1$, we have:

\begin{equation}
\sum_{k=1}^K p_k^{\text{target}} = 1 \implies \sum_{k=1}^K p_k^{\text{target}} = 1.
\end{equation}

Therefore,

\begin{equation}
\frac{\partial L_{\text{InfoNCA}}}{\partial s_\theta(y_i \mid x)} = - \left( p_i^{\text{target}} - p_i^{\text{model}} \cdot 1 \right ) = p_i^{\text{model}} - p_i^{\text{target}}.
\end{equation}

\end{proof}

\paragraph{Gradient of Weighted Contrastive Loss}

\begin{lemma}
The gradient of the weighted contrastive loss with respect to the model logits $s_\theta(y_i \mid x)$ is:

\begin{equation}
\frac{\partial L_{\text{weighted}}}{\partial s_\theta(y_i \mid x)} = p_i^{\text{weighted}} - p_i^{\text{pos}}
\end{equation}

where:

\begin{equation}
p_i^{\text{weighted}} = \frac{w_i e^{s_\theta(y_i \mid x)}}{\sum_{j=1}^K w_j e^{s_\theta(y_j \mid x)}},
\qquad\qquad
p_i^{\text{pos}} = \frac{w_i e^{s_\theta(y_i \mid x)}}{\sum_{k \in Y^+} w_k e^{s_\theta(y_k \mid x)}} \cdot \mathbb{I}_{y_i \in Y^+},
\end{equation}

and $\mathbb{I}_{y_i \in Y^+}$ is the indicator function, equal to $1$ if $y_i \in Y^+$ and $0$ otherwise.
\label{lem:gradient_weighted}
\end{lemma}

\begin{proof}
Let us denote:

\begin{equation}
A = \sum_{k \in Y^+} w_k e^{s_\theta(y_k \mid x)}, \quad Z = \sum_{j=1}^K w_j e^{s_\theta(y_j \mid x)}.
\end{equation}

The weighted contrastive loss is:

\begin{equation}
L_{\text{weighted}} = -\log \left( \frac{A}{Z} \right) = -\log A + \log Z.
\end{equation}

Compute the derivative with respect to $s_\theta(y_i \mid x)$:

\begin{equation}
\frac{\partial L_{\text{weighted}}}{\partial s_\theta(y_i \mid x)} = -\frac{1}{A} \frac{\partial A}{\partial s_\theta(y_i \mid x)} + \frac{1}{Z} \frac{\partial Z}{\partial s_\theta(y_i \mid x)}.
\end{equation}

Compute $\frac{\partial A}{\partial s_\theta(y_i \mid x)}$:

\begin{equation}
\frac{\partial A}{\partial s_\theta(y_i \mid x)} = w_i e^{s_\theta(y_i \mid x)} \cdot \mathbb{I}_{y_i \in Y^+}.
\end{equation}

Compute $\frac{\partial Z}{\partial s_\theta(y_i \mid x)}$:

\begin{equation}
\frac{\partial Z}{\partial s_\theta(y_i \mid x)} = w_i e^{s_\theta(y_i \mid x)}.
\end{equation}

Substitute back into the gradient:

\begin{align}
\frac{\partial L_{\text{weighted}}}{\partial s_\theta(y_i \mid x)} &= -\frac{1}{A} w_i e^{s_\theta(y_i \mid x)} \cdot \mathbb{I}_{y_i \in Y^+} + \frac{1}{Z} w_i e^{s_\theta(y_i \mid x)} \\
&= w_i e^{s_\theta(y_i \mid x)} \left( \frac{1}{Z} - \frac{\mathbb{I}_{y_i \in Y^+}}{A} \right).
\end{align}

Recognize that:

\begin{equation}
p_i^{\text{weighted}} = \frac{w_i e^{s_\theta(y_i \mid x)}}{Z}, \quad p_i^{\text{pos}} = \frac{w_i e^{s_\theta(y_i \mid x)}}{A} \cdot \mathbb{I}_{y_i \in Y^+}.
\end{equation}

Therefore:

\begin{equation}
\frac{\partial L_{\text{weighted}}}{\partial s_\theta(y_i \mid x)} = p_i^{\text{weighted}} - p_i^{\text{pos}}.
\end{equation}

Since $p_i^{\text{pos}} = 0$ when $y_i \notin Y^+$, we have:

\begin{equation}
\frac{\partial L_{\text{weighted}}}{\partial s_\theta(y_i \mid x)} = \begin{cases}
p_i^{\text{weighted}} - p_i^{\text{pos}}, & \text{if } y_i \in Y^+, \\
p_i^{\text{weighted}} - 0 = p_i^{\text{weighted}}, & \text{if } y_i \in Y^-.
\end{cases}
\end{equation}

However, this suggests that the gradient is always positive for negative examples. In other words, given $w_i$ and $Z$ are positive, $e^{s_\theta(y_i \mid x)}$ keeps increasing. But note that $s_\theta(y_i \mid x) = - \log\left(P_\theta(y_i \mid x)\right)$. Hence $\frac{1}{P_\theta(y_i \mid x)}$ keeps increasing implying that $P_\theta(y_i \mid x)$ keeps decreasing. i.e. at the stationary point, $P_\theta(y_i \mid x) \to 0$ for all negative examples, $y_i\in Y^-$.

Now let us examine the positive examples. The gradient simplifies to $w_i e^{s_\theta(y_i \mid x)} \left(\frac{1}{Z} - \frac{1}{A}\right)$. Since $Z \geq A$, $\frac{1}{Z} \leq \frac{1}{A}$. Hence the gradient term with respect to $s_\theta(y_i \mid x)$ is negative.
Notice that $e^{s_\theta(y_i \mid x)} = \frac{1}{P_\theta(y_i \mid x)}$. A negative gradient implies that $\frac{1}{P_\theta(y_i \mid x)}$ decreases, implying that $P_\theta(y_i \mid x)$ increases for all positive examples $y_i\in Y^+$.

\end{proof}

We now provide the gradients directly in terms of $P_\theta(y_j \mid x)$ instead of the scores $s_\theta(y_j \mid x)$, for easy interpretibility in terms of the probabilities.

\begin{lemma}
\label{lem:gradient_wrt_prob}
Let the weighted contrastive loss be defined as:
\[
L_{\text{weighted}} = -\log \left( \frac{V}{U} \right ) = -\log V + \log U,
\]
where
\[
U = \sum_{j=1}^K u_j P_\theta(y_j \mid x), \quad V = \sum_{i \in Y^+} u_i P_\theta(y_i \mid x),
\]
and
\[
u_i = \frac{w_i}{P_{\text{ref}}(y_i \mid x)},
\]
with \( w_i = e^{\alpha \delta_i} \), \( P_\theta(y_i \mid x) \) being the model probability for response \( y_i \), and \( P_{\text{ref}}(y_i \mid x) \) being the reference model probability.

Then, the gradient of the weighted contrastive loss with respect to \( P_\theta(y_i \mid x) \) is given by:

\begin{itemize}
    \item For positive examples (\( y_i \in Y^+ \)):
    \begin{equation}
    \frac{\partial L_{\text{weighted}}}{\partial P_\theta(y_i \mid x)} = u_i \left( \frac{1}{U} - \frac{1}{V} \right ),
    \end{equation}
    \item For negative examples (\( y_i \notin Y^+ \)):
    \begin{equation}
    \frac{\partial L_{\text{weighted}}}{\partial P_\theta(y_i \mid x)} = \frac{u_i}{U}.
    \end{equation}
\end{itemize}
\end{lemma}

\begin{proof}
Using the score function \( s_\theta(y_i \mid x) = \log \left( \frac{P_\theta(y_i \mid x)}{P_{\text{ref}}(y_i \mid x)} \right ) \), we have \( e^{s_\theta(y_i \mid x)} = \dfrac{P_\theta(y_i \mid x)}{P_{\text{ref}}(y_i \mid x)} \).

The weighted contrastive loss becomes:
\[
L_{\text{weighted}} = -\log \left( \frac{\displaystyle \sum_{i \in Y^+} w_i e^{s_\theta(y_i \mid x)}}{\displaystyle \sum_{j=1}^K w_j e^{s_\theta(y_j \mid x)}} \right ) = -\log \left( \frac{\displaystyle \sum_{i \in Y^+} w_i \dfrac{P_\theta(y_i \mid x)}{P_{\text{ref}}(y_i \mid x)}}{\displaystyle \sum_{j=1}^K w_j \dfrac{P_\theta(y_j \mid x)}{P_{\text{ref}}(y_j \mid x)}} \right ) = -\log \left( \frac{V}{U} \right ),
\]
where \( u_i = \dfrac{w_i}{P_{\text{ref}}(y_i \mid x)} \), \( V = \sum_{i \in Y^+} u_i P_\theta(y_i \mid x) \), and \( U = \sum_{j=1}^K u_j P_\theta(y_j \mid x) \).

We compute the gradient of \( L_{\text{weighted}} \) with respect to \( P_\theta(y_i \mid x) \):
\[
\frac{\partial L_{\text{weighted}}}{\partial P_\theta(y_i \mid x)} = -\frac{1}{V} \cdot \frac{\partial V}{\partial P_\theta(y_i \mid x)} + \frac{1}{U} \cdot \frac{\partial U}{\partial P_\theta(y_i \mid x)}.
\]

\textbf{Case 1:} For \( y_i \in Y^+ \):
\[
\frac{\partial V}{\partial P_\theta(y_i \mid x)} = u_i, \quad \frac{\partial U}{\partial P_\theta(y_i \mid x)} = u_i.
\]
Thus,
\[
\frac{\partial L_{\text{weighted}}}{\partial P_\theta(y_i \mid x)} = -\frac{u_i}{V} + \frac{u_i}{U} = u_i \left( \frac{1}{U} - \frac{1}{V} \right ).
\]

\textbf{Case 2:} For \( y_i \notin Y^+ \):
\[
\frac{\partial V}{\partial P_\theta(y_i \mid x)} = 0, \quad \frac{\partial U}{\partial P_\theta(y_i \mid x)} = u_i.
\]
Thus,
\[
\frac{\partial L_{\text{weighted}}}{\partial P_\theta(y_i \mid x)} = 0 + \frac{u_i}{U} = \frac{u_i}{U}.
\]
\end{proof}

\begin{corollary}
The sign of the gradient indicates the optimization direction:

\begin{itemize}
    \item For positive examples (\( y_i \in Y^+ \)), since \( V \leq U \), we have \( \dfrac{1}{U} - \dfrac{1}{V} \leq 0 \). Therefore, the gradient \( \dfrac{\partial L_{\text{weighted}}}{\partial P_\theta(y_i \mid x)} \leq 0 \), and minimizing \( L_{\text{weighted}} \) involves \textbf{increasing} \( P_\theta(y_i \mid x) \).
    \item For negative examples (\( y_i \notin Y^+ \)), the gradient \( \dfrac{\partial L_{\text{weighted}}}{\partial P_\theta(y_i \mid x)} > 0 \), and minimizing \( L_{\text{weighted}} \) involves \textbf{decreasing} \( P_\theta(y_i \mid x) \).
\end{itemize}
\end{corollary}

\begin{proof}
As established in the lemma:

\textbf{For positive examples (\( y_i \in Y^+ \))}:
Since \( V = \sum_{i \in Y^+} u_i P_\theta(y_i \mid x) \) and \( U = V + \sum_{j \notin Y^+} u_j P_\theta(y_j \mid x) \), it follows that \( V \leq U \) and thus \( \dfrac{1}{U} - \dfrac{1}{V} \leq 0 \).

Therefore, the gradient:
\[
\frac{\partial L_{\text{weighted}}}{\partial P_\theta(y_i \mid x)} = u_i \left( \frac{1}{U} - \frac{1}{V} \right ) \leq 0.
\]
A negative gradient indicates that increasing \( P_\theta(y_i \mid x) \) will decrease \( L_{\text{weighted}} \). Hence, to minimize the loss, we should \textbf{increase} \( P_\theta(y_i \mid x) \) for positive examples.

\textbf{For negative examples (\( y_i \notin Y^+ \))}:
The gradient is:
\[
\frac{\partial L_{\text{weighted}}}{\partial P_\theta(y_i \mid x)} = \frac{u_i}{U} > 0,
\]
since \( u_i > 0 \) and \( U > 0 \). A positive gradient indicates that decreasing \( P_\theta(y_i \mid x) \) will decrease \( L_{\text{weighted}} \). Therefore, to minimize the loss, we should \textbf{decrease} \( P_\theta(y_i \mid x) \) for negative examples.
\end{proof}
\section{Characterization of Stationary Points}
\label{sec:stationary_point_characterization}

We now characterize the stationary points of both loss functions.

\subsection{Stationary Points of the InfoNCA Loss Function}
\begin{theorem}
\label{theorem:infonca_stationary}
For the InfoNCA loss, the stationary points occur when:

\begin{equation}
p_i^{\text{model}} = p_i^{\text{target}}, \quad \forall i \in \{1, \dots, K\}.
\end{equation}
\end{theorem}

\begin{proof}
Stationary points are defined by the condition:

\begin{equation}
\frac{\partial L_{\text{InfoNCA}}}{\partial s_\theta(y_i \mid x)} = 0, \quad \forall i.
\end{equation}

From the gradient:

\begin{equation}
\frac{\partial L_{\text{InfoNCA}}}{\partial s_\theta(y_i \mid x)} = p_i^{\text{model}} - p_i^{\text{target}},
\end{equation}

setting the gradient to zero yields:

\begin{equation}
p_i^{\text{model}} = p_i^{\text{target}}, \quad \forall i.
\end{equation}
\end{proof}

\begin{remark}
This stationary point is suboptimal because \(p_i^{\text{model}} \) expands to:

\begin{equation*}
p_i^{\text{model}} = \frac{e^{\log P_\theta(y_i|x) - \log P_{\text{ref}}(y_i|x)}}{\sum_{j=1}^K e^{\log P_\theta(y_j|x) - \log P_{\text{ref}}(y_j|x)}}
\end{equation*}

Rather than equating the soft-max of the difference between \(\log P_\theta(y_i|x)\) and \(\log P_{\text{ref}}(y_i|x)\) to \(p_i^{\text{target}} \), optimality may require directly setting \(\log P_\theta(y|x)\) to match the softmax of the target scores.
\end{remark}

\subsection{Stationary Points of the Weighted Contrastive Loss under Simplifying Assumptions}
\label{sec:stationary_point_simplified}
\begin{lemma}
Consider the weighted contrastive loss function in a simplified scenario with the following conditions: There are \( N^+ \) positive examples, each with weight \( w^+ \), and \( N^- \) negative examples, each with weight \( w^- \). All positive examples have the same score \( s^{(t)} \) at iteration \( t \), and all negative examples have the same score \( s^{(t)} \) at iteration \( t \). Then, the update rule for the score \( s^{(t)} \) of the positive examples at iteration \( t+1 \) is given by
\begin{equation}
s^{(t+1)} = s^{(t)} + \eta \left( \frac{ N^- w^- }{ N^+ ( N^+ w^+ + N^- w^- ) } \right),
\end{equation}
where \( \eta \) is the learning rate.
\end{lemma}

\begin{proof}

Let \(Y^+ \) denote the set of positive examples and \(Y^-\) the set of negative examples, with \(N^+\) and \(N^-\) examples respectively for a total of \(K = N^+ + N^-\) examples. With weights \(w^+\) and \(w^-\) assigned to positive and negative examples respectively, and logits \(s^{(t)}\) for both classes at timestep \(t\), the weighted contrastive loss function is defined as:

\begin{equation}
L_{\text{weighted}}(\theta) = -\log \left( \frac{ \sum\limits_{i \in Y^+} w_i e^{s_i} }{ \sum\limits_{j=1}^{K} w_j e^{s_j} } \right),
\end{equation}

where \(w_i = w^+\) and \(s_i = s^{(t)}\) for \(i \in Y^+\), and \(w_j = w^-\) and \(s_j = s^{(t)}\) for \(j \in Y^-\).

Compute the numerator \( A \) and the denominator \( Z \) of the loss function:
\begin{align}
A &= \sum_{i \in Y^+} w_i e^{s_i} = N^+ w^+ e^{s^{(t)}}, \\
Z &= \sum_{j=1}^{K} w_j e^{s_j} = N^+ w^+ e^{s^{(t)}} + N^- w^- e^{s^{(t)}} = e^{s^{(t)}} ( N^+ w^+ + N^- w^- ).
\end{align}

For positive examples \( i \in Y^+ \), the weighted probability \( p_i^{\text{weighted}} \) and the positive probability \( p_i^{\text{pos}} \) are:
\begin{align}
p_i^{\text{weighted}} &= \frac{ w^+ e^{s^{(t)}} }{ Z } = \frac{ w^+ }{ N^+ w^+ + N^- w^- }, \\
p_i^{\text{pos}} &= \frac{ w^+ e^{s^{(t)}} }{ A } = \frac{ w^+ }{ N^+ w^+ } = \frac{1}{ N^+ }.
\end{align}

The gradient of the loss with respect to \( s^{(t)} \) for positive examples is:
\begin{equation}
\frac{ \partial L_{\text{weighted}} }{ \partial s^{(t)} } = p_i^{\text{weighted}} - p_i^{\text{pos}} = \frac{ w^+ }{ N^+ w^+ + N^- w^- } - \frac{1}{ N^+ }.
\end{equation}

To simplify this expression, we find a common denominator \( D = N^+ ( N^+ w^+ + N^- w^- ) \):
\begin{align}
\frac{ \partial L_{\text{weighted}} }{ \partial s^{(t)} } &= \frac{ w^+ N^+ - ( N^+ w^+ + N^- w^- ) }{ D } \\
&= \frac{ w^+ N^+ - N^+ w^+ - N^- w^- }{ N^+ ( N^+ w^+ + N^- w^- ) } \\
&= \frac{ - N^- w^- }{ N^+ ( N^+ w^+ + N^- w^- ) }.
\end{align}

The update rule for \( s^{(t)} \) is then:
\begin{equation}
s^{(t+1)} = s^{(t)} - \eta \frac{ \partial L_{\text{weighted}} }{ \partial s^{(t)} } = s^{(t)} + \eta \left( \frac{ N^- w^- }{ N^+ ( N^+ w^+ + N^- w^- ) } \right).
\end{equation}
This completes the proof.
\end{proof}

\begin{corollary}
Assuming the initial scores are zero (\( s^{(0)} = 0 \)), the score \( s^{(t)} \) of the positive examples at iteration \( t \) is given by
\begin{equation}
s^{(t)} = t \eta \left( \frac{ N^- w^- }{ N^+ ( N^+ w^+ + N^- w^- ) } \right).
\end{equation}
\end{corollary}

\begin{proof}
From the update rule established in the lemma,
\begin{equation}
s^{(t+1)} = s^{(t)} + c,
\end{equation}
where
\begin{equation}
c = \eta \left( \frac{ N^- w^- }{ N^+ ( N^+ w^+ + N^- w^- ) } \right).
\end{equation}
Since \( s^{(0)} = 0 \), we have
\begin{align}
s^{(1)} &= s^{(0)} + c = c, \\
s^{(2)} &= s^{(1)} + c = 2c, \\
&\ \ \vdots \\
s^{(t)} &= t c.
\end{align}
Substituting \( c \) back into the expression, we obtain
\begin{equation}
s^{(t)} = t \eta \left( \frac{ N^- w^- }{ N^+ ( N^+ w^+ + N^- w^- ) } \right).
\end{equation}
\end{proof}

\begin{corollary}
In the special case where there is one positive example (\( N^+ = 1 \)) and one negative example (\( N^- = 1 \)), and the weights are \( w^+ = w^- = 1 \) (as in Direct Preference Optimization), the score \( s^{(t)} \) at iteration \( t \) is:
\begin{equation}
s^{(t)} = \frac{ \eta t }{ 2 }.
\end{equation}
\end{corollary}

\begin{proof}
Substituting \( N^+ = N^- = 1 \) and \( w^+ = w^- = 1 \) into the expression for \( s^{(t)} \):
\begin{align}
s^{(t)} &= t \eta \left( \frac{ 1 \times 1 }{ 1 \times (1 \times 1 + 1 \times 1) } \right ) \\
&= t \eta \left( \frac{1}{1 \times (1 + 1)} \right ) \\
&= t \eta \left( \frac{1}{2} \right ) \\
&= \frac{ \eta t }{ 2 }.
\end{align}
\end{proof}

\begin{lemma}
Consider the general case where positive examples may have different weights \( w_i^+ \), and each positive example \( i \) has its own score \( s_i^{(t)} \) at iteration \( t \). Assuming initial scores \( s_i^{(0)} = 0 \) for all positive examples, the score \( s_i^{(t)} \) of positive example \( i \) at iteration \( t \), up to a linear approximation, is given by
\begin{equation}
s_i^{(t)} = t \eta w_i^+ \left( \frac{ B_0 }{ A_0 Z_0 } \right ),
\end{equation}
where \( A_0 = \sum_{k \in Y^+} w_k^+ \), \( B_0 = \sum_{j \in Y^-} w_j^- \), \( Z_0 = A_0 + B_0 \), and \( \eta \) is the learning rate.
\end{lemma}

\begin{proof}
At iteration \( t = 0 \), the initial scores are \( s_i^{(0)} = 0 \) for all \( i \in Y^+ \). The sums are:
\begin{align}
A_0 &= \sum_{k \in Y^+} w_k^+ e^{s_k^{(0)}} = \sum_{k \in Y^+} w_k^+ = W^+, \\
B_0 &= \sum_{j \in Y^-} w_j^- e^{s_j^{(0)}} = \sum_{j \in Y^-} w_j^- = W^-.
\end{align}
The total sum is \( Z_0 = A_0 + B_0 = W^+ + W^- \).

The gradient for each positive example \( i \) at \( t = 0 \) is:
\begin{equation}
\frac{ \partial L_{\text{weighted}} }{ \partial s_i^{(0)} } = - w_i^+ e^{s_i^{(0)}} \left( \frac{ B_0 }{ A_0 Z_0 } \right ) = - w_i^+ \left( \frac{ B_0 }{ A_0 Z_0 } \right ).
\end{equation}

The update rule is:
\begin{equation}
s_i^{(1)} = s_i^{(0)} - \eta \frac{ \partial L_{\text{weighted}} }{ \partial s_i^{(0)} } = \eta w_i^+ \left( \frac{ B_0 }{ A_0 Z_0 } \right ).
\end{equation}

Assuming that the term \( \frac{ B_0 }{ A_0 Z_0 } \) remains approximately constant over iterations (which holds when \( \eta \) is small and changes in \( s_i^{(t)} \) are small), the score at iteration \( t \) is:
\begin{equation}
s_i^{(t)} = t \eta w_i^+ \left( \frac{ B_0 }{ A_0 Z_0 } \right ).
\end{equation}
\end{proof}

\begin{remark}
The approximation assumes that \( A_t \), \( B_t \), and \( Z_t \) remain close to their initial values \( A_0 \), \( B_0 \), and \( Z_0 \) over the iterations considered, and the score values remain small. This is reasonable for small learning rates \( \eta \) and a limited number of iterations \( t \).
\end{remark}

\subsection{Stationary Points of the Weighted Contrastive Loss}

We now analyze the stationary points of our weighted contrastive loss function.

\begin{lemma}
\label{lem:weighted_stationary}
For the weighted contrastive loss function, the stationary point occurs when the probabilities of the negative samples approach zero, i.e.,
\begin{equation}
P_\theta(y_i \mid x) \to 0 \quad \text{for all } y_i \in Y^-.
\end{equation}
\end{lemma}

\begin{proof}
From Lemma~\ref{lem:gradient_weighted}, the gradient of the weighted contrastive loss with respect to the model logits $s_\theta(y_i \mid x)$ is:
\begin{equation}
\frac{\partial L_{\text{weighted}}}{\partial s_\theta(y_i \mid x)} = \begin{cases}
p_i^{\text{weighted}} - p_i^{\text{pos}}, & \text{if } y_i \in Y^+, \\
p_i^{\text{weighted}}, & \text{if } y_i \in Y^-.
\end{cases}
\end{equation}

At a stationary point, the gradient must be zero for all $y_i$. Consider the negative samples $y_i \in Y^-$. Setting the gradient to zero yields:
\begin{equation}
\frac{\partial L_{\text{weighted}}}{\partial s_\theta(y_i \mid x)} = p_i^{\text{weighted}} = 0.
\end{equation}

Since $p_i^{\text{weighted}}$ is the normalized weighted probability of $y_i$, given by:
\begin{equation}
p_i^{\text{weighted}} = \frac{w_i e^{s_\theta(y_i \mid x)}}{\sum_{j=1}^K w_j e^{s_\theta(y_j \mid x)}},
\end{equation}
and $w_i > 0$, the only way for $p_i^{\text{weighted}}$ to be zero is if $e^{s_\theta(y_i \mid x)} = 0$, which implies:
\begin{equation}
s_\theta(y_i \mid x) \to -\infty \quad \implies \quad P_\theta(y_i \mid x) \to 0 \quad \text{for } y_i \in Y^-.
\end{equation}

Similarly, for positive samples $y_i \in Y^+$, the gradient is:
\begin{equation}
\frac{\partial L_{\text{weighted}}}{\partial s_\theta(y_i \mid x)} = p_i^{\text{weighted}} - p_i^{\text{pos}} = 0.
\end{equation}

This implies:
\begin{equation}
p_i^{\text{weighted}} = p_i^{\text{pos}}.
\end{equation}

Since the probabilities of the negative samples approach zero, the denominator in $p_i^{\text{weighted}}$ becomes:
\begin{equation}
\sum_{j=1}^K w_j e^{s_\theta(y_j \mid x)} \approx \sum_{k \in Y^+} w_k e^{s_\theta(y_k \mid x)}.
\end{equation}

Therefore, $p_i^{\text{weighted}} \approx p_i^{\text{pos}}$, satisfying the condition for the gradient to be zero for positive samples.

Thus, at the stationary point, the probabilities of the negative samples approach zero.
\end{proof}

\begin{remark}
When the probabilities of the negative samples approach zero, the scores $s_\theta(y_i \mid x)$ for $y_i \in Y^-$ tend to $-\infty$. Since:
\begin{equation}
e^{s_\theta(y_i \mid x)} = \frac{P_\theta(y_i \mid x)}{P_{\text{ref}}(y_i \mid x)} \to 0,
\end{equation}
the weighted contributions of the negative samples to the numerator and denominator of $L_{\text{weighted}}$ become negligible.

Consequently, the numerator and denominator of $L_{\text{weighted}}$ become equal:
\begin{equation}
\sum_{i \in Y^+} w_i e^{s_\theta(y_i \mid x)} \approx \sum_{j=1}^K w_j e^{s_\theta(y_j \mid x)}.
\end{equation}

Therefore:
\begin{equation}
L_{\text{weighted}} = -\log \left( \frac{\sum\limits_{i \in Y^+} w_i e^{s_\theta(y_i \mid x)}}{\sum\limits_{j=1}^K w_j e^{s_\theta(y_j \mid x)}} \right) \approx -\log 1 = 0.
\end{equation}

This implies that the loss vanishes when the probabilities of the negative samples approach zero, indicating that the model has successfully minimized the loss by focusing entirely on the positive responses.
\end{remark}

\section{Additional Experiments}
\label{app:additional_experiments}

% \begin{figure}[!h]               % one float, two columns inside
%   \centering
%   \begin{minipage}[t]{0.5\textwidth}
%     \centering
%     \includegraphics[width=\linewidth]{images/preference_optimization_samp_temp_variation.pdf }
%     \vspace{-0.1in}
%     \captionof{figure}{Effect of Sampling Temperature on Policy Optimization Techniques. \mytodo{make font sizes in the fig much larger}}% for Mistral-Base (7B) on the AlpacaEval 2 Benchmark: (a) Length-Controlled Win Rate (LC) and (b) Overall Win Rate (WR).}
%     \label{fig:samp-temp-analysis}
%   \end{minipage}\hfill          % \hfill inserts a flexible gutter
%   \begin{minipage}[t]{0.52\textwidth}
%     \paragraph{Optimizing for lower temperatures}

\paragraph{Optimizing for Lower Temperatures}
\begin{figure}[!tbph]
    \centering
    % \vspace{-0.25in}  % Adjust vertical space before the figure
    \includegraphics[width=.6\linewidth]{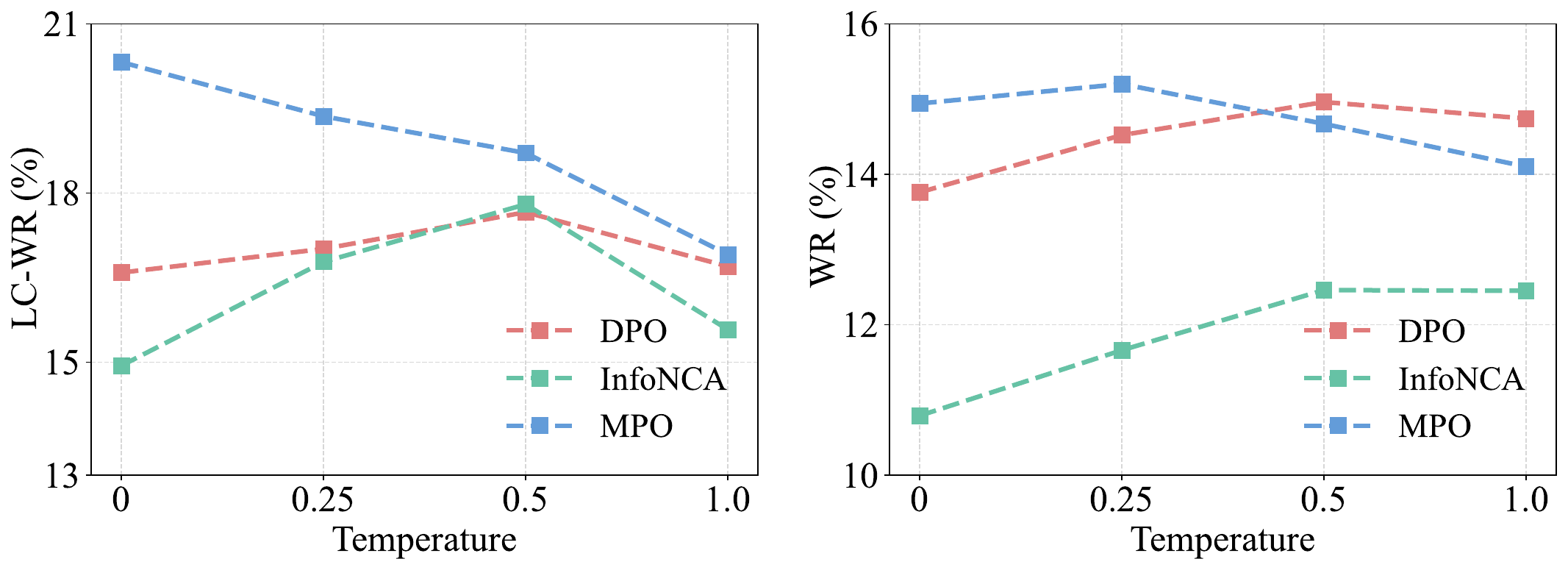}
    % \vspace{-12pt}  % Adjust vertical space after the figure
    \caption{Performance variation of $\mpo$ for Mistral-Base with different sampling temperature on AlpacaEval2.}
    \label{fig:samp-temp-analysis-appendix}
\end{figure}
Fig. \ref{fig:samp-temp-analysis-appendix} evaluates MPO, DPO, and InfoNCA on AlpacaEval 2 (Mistral-Base 7B) against sampling temperature in the off-policy setting for Mistral-7B. MPO achieves its peak LC-WR of $\sim 20.5\%$ at zero temperature (greedy decoding), declining thereafter. Notably, MPO consistently leads in LC-Win Rate across all temperatures.

\begin{figure}[!tbph]
    \centering
    % \vspace{-0.25in}  % Adjust vertical space before the figure
    \includegraphics[width=.4\linewidth]{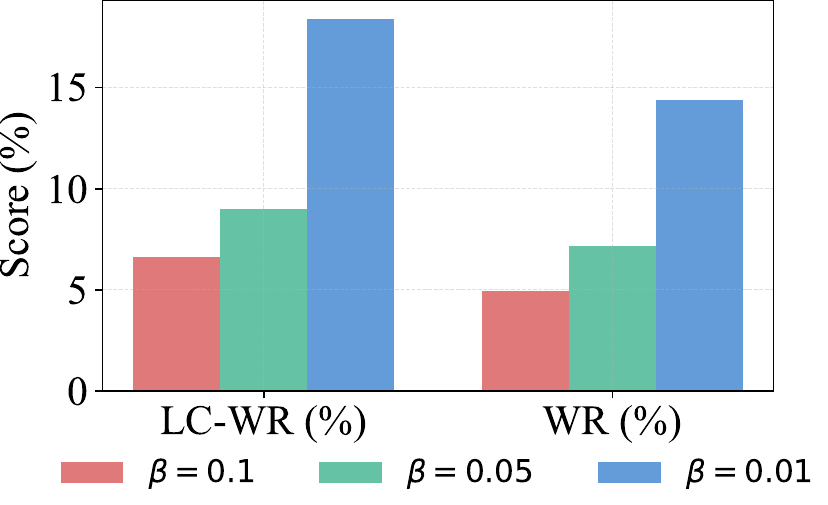}
    % \vspace{-12pt}  % Adjust vertical space after the figure
    \caption{Performance variation of $\mpo$ for Mistral-Base with different $\beta$ values on AlpacaEval2.}
    \label{fig:beta-analysis}
\end{figure}

\paragraph{Effect of $\beta$ on Optimization Dynamics.}

In Figure \ref{fig:beta-analysis}, we conduct an ablation study to investigate the role of the temperature parameter $\beta$ in our optimization objective. As $\beta$ increases, the gradient magnitudes decrease due to the flattening of the sigmoid-like objective, effectively pushing the scaled reward term $\beta \cdot \log \frac{\pi_{\textit{policy}}(y|x)}{\pi_{\textit{ref}}(y|x)}$ outside the region where the derivative is steepest. This leads to slower updates and consequently a more conservative learning trajectory. In contrast, smaller values of $\beta$ keep the optimization in the high-sensitivity regime of the sigmoid, allowing for faster convergence. Empirically, we observe that using a smaller $\beta$ (e.g., $\beta=0.01$) enables the model to reach comparable performance to higher $\beta$ settings with significantly fewer epochs and smaller learning rates. This demonstrates that careful tuning of $\beta$ can not only preserve final accuracy but also substantially reduce training cost.

\begin{figure}[!tbph]
    \centering
    % \vspace{-0.25in}  % Adjust vertical space before the figure
    \includegraphics[width=.35\linewidth]{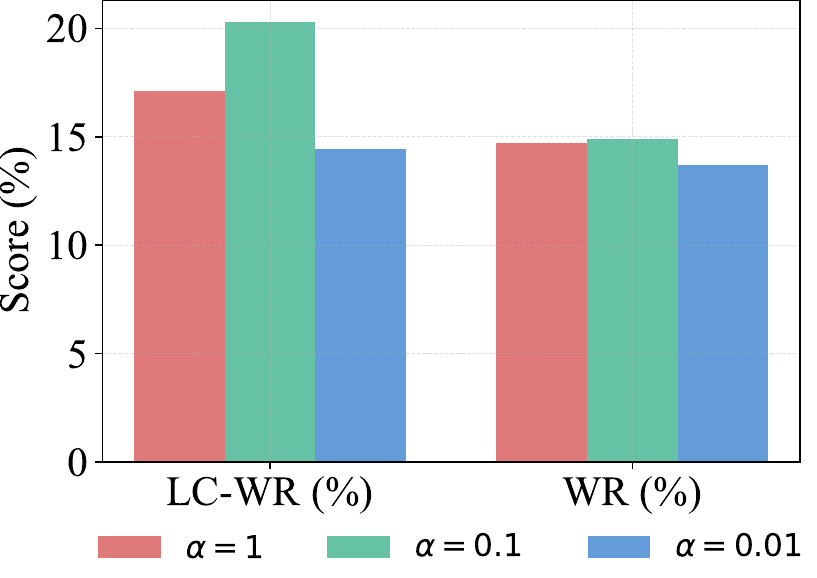}
    % \vspace{-12pt}  % Adjust vertical space after the figure
    \caption{Performance variation of W-$\mpo$ for Mistral-Base with different $\alpha$ values on AlpacaEval2.}
    \label{fig:alpha-analysis}
\end{figure}

\paragraph{Effect of $\alpha$ on Optimization Dynamics.}
In Figure \ref{fig:alpha-analysis}, we analyze the impact of the reward normalization coefficient $\alpha$, which scales the raw reward signal before computing the deviation weight in W-$\mpo$. This parameter plays a crucial role in balancing the sensitivity of the weighting function. A smaller $\alpha$ (e.g., $\alpha=1$) retains the raw reward's expressivity, but can lead to high variance in gradient updates, potentially destabilizing training. On the other hand, a very large $\alpha$ (e.g., $\alpha=0.01$) overly compresses the reward differences, weakening the contrast between preferred and dispreferred responses—this leads to underutilized supervision and poor alignment, as evidenced by the drop in LC-WR from 20.3 to 14.43. We find that an intermediate setting (e.g., $\alpha=0.1$) strikes a favorable balance, providing stable optimization while preserving informative contrast. This setting achieves the highest LC-WR and WR performance, demonstrating the importance of calibrating $\alpha$ to match the scale and distribution of reward scores.

\section{Baselines Used for Comparison} 
\label{sec:baselines}

When dealing with reward datasets where each instruction has more than two
$n>2$ responses, one common approach is to convert the data into pairwise preferences and then apply preference optimization techniques such as Direct Preference Optimization (DPO). Several strategies can be adopted for this purpose, each offering distinct trade-offs in terms of dataset richness and computational overhead.One straightforward method, as implemented by Zephyr \cite{tunstall2023zephyr}, involves selecting the response with the highest reward and pairing it with a randomly chosen response from the remaining responses for each instruction. Another variant involves pairing the highest-rewarded response with the lowest-rewarded response for each instruction, ensuring a clear distinction between preferences.Additionally, alternative baselines can be explored to enhance performance by incorporating more suboptimal responses during training. By applying DPO to combinations of responses, we can significantly expand the preference dataset and potentially achieve improved optimization. Baselines in this context are:

\textbf{DPOx$\binom{N}{2}$}:  
In this approach, all possible pairwise combinations of $\binom{N}{2}$ are generated, and DPO is applied to the entire combinatorial dataset. This method ensures the model is exposed to a comprehensive range of preference relationships, including those involving suboptimal responses.

% \paragraph{Baselines}

We implement and evaluated several preference optimization baselines under a consistent offline setting using the \textsc{UltraFeedback} dataset. Below, we describe each method in detail:

\textbf{Group Relative Policy Optimization (GRPO)}: We implement GRPO in an offline setup, where both the response set (four responses per prompt) and their scalar rewards are fixed, as provided by \textsc{UltraFeedback}. GRPO uses a relative advantage formulation that weighs each response based on its normalized reward gap.

\textbf{Plackett-Luce Ranking Model}: We implement the Plackett-Luce model, in which the ranking over responses is determined by the reward scores provided in \textsc{UltraFeedback}. The objective maximizes the likelihood of the observed ranking via a sequential softmax over score-ordered responses.

\textbf{InfoNCA}: We adapt InfoNCA to the preference optimization setting by treating each response as a sample in a softmax-based contrastive loss, where rewards from \textsc{UltraFeedback} define the similarity scores. All responses are used jointly during training.

\textbf{$\mpo$-1vsk}: In this baseline for each prompt, the response with the highest reward is selected as the positive sample, and the remaining responses form the negative set. This formulation enables multi-response contrastive alignment.

All baselines are trained under a unified offline pipeline to ensure fair and consistent comparison across evaluation benchmarks.

\begin{table}[t]
\centering
\resizebox{\textwidth}{!}{
\begin{tabular}{@{}lccc|ccc@{}}
\toprule
\multirow{2}{*}{\textbf{Method}} 
& \multicolumn{3}{c|}{\textbf{Llama3-8B-Instruct}} 
& \multicolumn{3}{c}{\textbf{Mistral-7B-Instruct}} \\
& \textbf{LC-WR (\%)} & \textbf{WR (\%)} & \textbf{Std. Err.} 
& \textbf{LC-WR (\%)} & \textbf{WR (\%)} & \textbf{Std. Err.} \\
\midrule
DPO   & 41.1 & 37.8 & 1.38 & 26.8 & 24.6 & 1.39 \\
MPO   & 49.0 & 50.6 & 1.52 & 28.2 & 29.4 & 1.34 \\
W-MPO & \textbf{52.1} & \textbf{52.5} & 1.46 & \textbf{29.5} & \textbf{30.6} & 1.45 \\
\bottomrule
\end{tabular}
}
\caption{Comparison of $\mpo$ variants with $\dpo$ on \textbf{Llama3-8B-Instruct} and \textbf{Mistral-7B-Instruct}. We report Length-Controlled Win Rate (LC-WR), Overall Win Rate (WR), and standard error across evaluations.}
\label{tab:mpo-dpo-llama-mistral-error}
\end{table}

% Other baselines, such as InfoNCA and NCA, suggest that naively applying DPO to combinations of responses may lead to suboptimal performance. They leveraged Noise Contrastive Estimation (NCE) to bridge the gap in handling reward datasets explicitly annotated with scalar evaluations. According to their findings, the theoretical guarantees they provide ensure convergence, which is not guaranteed when applying DPO in this manner.

\section{Experimental Details} 
\label{app:experimental_details}

\paragraph{Model and Training Setting}
For our off-policy experiments, we utilized the Ultrafeedback Dataset \cite{cui2023ultrafeedback}, an instruction-following benchmark annotated by GPT-4. This dataset consists of approximately 64,000 instructions, each paired with four responses generated by different language models. GPT-4 assigned scalar rewards on a 0-to-10 scale for each response, which prior research has shown to correlate strongly with human annotations. This establishes GPT-4 ratings as a reliable and cost-efficient alternative to manual feedback.

In our broader framework, we first trained a base model (\href{https://huggingface.co/mistralai/Mistral-7B-v0.1}{mistralai/Mistral-7B-v0.1}) and (\href{https://huggingface.co/meta-llama/Meta-Llama-3-8B}{meta-llama/Meta-Llama-3-8B}) on the UltraChat-200k dataset to obtain an SFT model. This SFT model, trained on open-source data, provides a transparent starting point. Subsequently, we refined the model by performing preference optimization on the UltraFeedback dataset. Once fine-tuned, the model was used for alignment. This two-step process ensures the model is well-prepared for tasks.

In our experiments, we observed that tuning hyperparameters is critical for optimizing the performance of all offline preference optimization algorithms, including DPO, SimPO, $\mpo$ and W-$\mpo$. Hyperparameter selection is very crucial for these.For $\mpo$, we found that setting the $\beta$ parameter in the range of 0.01 to 0.05 consistently yields strong performance. For W-$\mpo$, we found that setting the $\beta$ parameter in the range of 0.01 to 0.05 and $\alpha$ to 10 consistently yields strong performance. Learning rate for mistral-experiments was fixed to 3e-7 and whereas for llama it was 5e-7 for all the baselines.

For our on-policy and iterative experiments, we utilize a pretrained instruction-tuned model (\href{https://huggingface.co/meta-llama/Meta-Llama-3-8B-Instruct}{meta-llama/MetaLlama-3-8B-Instruct}) and (\href{https://huggingface.co/mistralai/Mistral-7B-Instruct-v0.2}{mistralai/Mistral-7B-Instruct-v0.2}) , as the SFT model. These models have undergone extensive instruction tuning, making them more capable and robust compared to the SFT models used in the Base setup. However, their reinforcement learning with human feedback (RLHF) procedures remain undisclosed, making them less transparent.

\begin{wraptable}{r}{0.6\columnwidth}
\vspace{-10pt}
\centering
\caption{Hyperparameters $\beta$ and learning rate $\alpha$ across different training settings.}
\label{tab:hyperparams}
\begin{tabular}{@{}lllll@{}}
\toprule
\textbf{Setting} & \textbf{Model} & \textbf{Iter} & $\beta$ & $\textit{lr}$ \\
\midrule
\multirow{2}{*}{Offline} 
  & Mistral-Base (7B) & -- & 0.01 & 4e-7 \\
  & Llama3-Base (8B)   & -- & 0.01 & 5e-7 \\
\midrule
\multirow{2}{*}{Online}  
  & Mistral-Instruct (7B) & -- & 0.01 & 1.5e-7 \\
  & LLama3-Instruct (8B)   & -- & 0.01 & 3e-7 \\
\midrule
\multirow{6}{*}{Iterative}  
  & Mistral-Instruct (7B) & 1 & 0.01 & 2.5e-7 \\
  & Llama3-Instruct (8B)   & 1 & 0.01 & 5e-7 \\
  & Mistral-Instruct (7B) & 2 & 0.01 & 2e-7 \\
  & Llama3-Instruct (8B)   & 2 & 0.01 & 4e-7 \\
  & Mistral-Instruct (7B) & 3 & 0.01 & 1.5e-7 \\
  & Llama3-Instruct (8B)   & 3 & 0.01 & 3e-7 \\
\bottomrule
\end{tabular}
\vspace{-10pt}
\end{wraptable}

To reduce distribution shift between the SFT models and the preference optimization process, we follow the approach in \cite{meng2024simpo} and generate the preference dataset using the same SFT models. This ensures that our setup is more aligned with an on-policy setting. Specifically, we utilize prompts from the UltraFeedback dataset \cite{cui2023ultrafeedback} and regenerate the resonses using the SFT models. For each prompt x, we produce 5 responses by sampling from the SFT model with a sampling temperature of 1.0 and top-p 1.0. We then use the reward model (\href{https://huggingface.co/Skywork/Skywork-Reward-Llama-3.1-8B-v0.2}{Skywork/Skywork-Reward-Llama-3.1-8B-v0.2}) \cite{liu2024skywork} to score all the responses

In our online and iterative training experiments, we found that systematic hyperparameter tuning is critical for optimizing performance across diverse datasets. Notably, the choice of the $\beta$ parameter had a substantial impact, with values in the range of 0.01 to 0.1 consistently yielding strong results. For all experiments, we adopted a fixed learning rate of 1.5e-7 for mistral-based models and 3.0e-7 for llama-based models.

All experiments employed the AdamW optimizer with a cosine annealing learning rate scheduler, incorporating a 0.01\% warmup step and a total of 1 training epoch. We used 8×A100 (80GB) GPUs under DeepSpeed with distributed data parallel (DDP) training. The global batch size was set to 128, with a per-device batch size of 2 and a gradient accumulation step of 8 to ensure memory efficiency and training stability.

These findings emphasize the importance of well-calibrated hyperparameters and infrastructure settings to achieve robust and reproducible outcomes in post-training alignment tasks.

\paragraph{Evaluation Benchmarks:} We evaluate our models using three widely recognized open-ended instruction-following benchmarks: MT-Bench \cite{zheng2023judging}, AlpacaEval 2, and Arena-Hard v0.1 \cite{zheng2023judging}. These benchmarks test the models' conversational versatility across a broad range of queries and are broadly utilized in the research community. AlpacaEval2 \cite{dubois2024length} includes 805 questions derived from five datasets, while MT-Bench spans eight categories with a total of 80 questions. Arena-Hard, a recently updated version of MT-Bench, focuses on 500 well-defined technical problem-solving queries. Scores are reported based on the evaluation protocols of each benchmark. For AlpacaEval2, both the raw win rate (WR) and the length-controlled win rate (LC) are reported, with LC specifically designed to mitigate biases related to model verbosity. For Arena-Hard, the win rate is reported relative to a baseline model. For MT-Bench, the average score is calculated using evaluations by GPT-4 as judge.
For decoding details, we generate responses using both greedy decoding and multinomial sampling with temperatures of 0.2, 0.5, and 1.0. To address potential biases introduced by multinomial sampling at varying temperatures, we generate responses three times for each setting at different seed and average their performance across the datasets.

\clearpage
\section{Reward Loss Computation}
\label{sec:reward_loss_computation}

In this section we provide the actual code used to compute the reward losses.
\vspace{0.0in}
\begin{lstlisting}[language=Python]
import torch

def mpo_loss(pi_logps, ref_logps, rewards, beta, alpha):
    """
    pi_logps: policy logprobs for N responses, shape (Batch_Size, N)
    ref_logps: reference logprobs for N responses, shape (Batch_Size, N)
    rewards: reward labels for N responses, shape (Batch_Size, N)
    beta: Temperature parameter for the MPO loss
    alpha: rating weight
    """
    
    logits = pi_logps - ref_logps  # Compute logits
    rewards = rewards / alpha # Normalizing the reward value to logits scale
    
    mean_rewards = torch.mean(rewards, dim=-1)
    if self.norm > 0:
        weights = torch.abs(rewards - mean_rewards.reshape(-1, 1))
    else:
        deviation_reward = 0
    
    pos_mask = (rewards > mean_rewards.reshape(-1, 1)) * 1
    neg_mask = torch.logical_not(pos_mask) * 1
    
    eps = 1e-10   
    logits = (logits + weights) * beta
    logits = logits - logits.max(dim=-1, keepdim=True)[0]  # Stabilize logits
    softmax_val = torch.softmax(logits + eps, dim=-1)
    pos_sum = torch.clamp(torch.sum(softmax_val * pos_mask, dim=-1), min=eps)
    neg_sum = torch.clamp(torch.sum(softmax_val * neg_mask, dim=-1), min=eps)
    
    losses = -1 * torch.log(pos_sum / (pos_sum + neg_sum + eps * 2))

    return losses.mean()
\end{lstlisting}

% Pastel blue for MPO

\clearpage
\vspace{0.0in}
\section{Examples}
\label{sec:examples}

In this section, we provide examples of user queries and responses from our model and the base DPO model.

\begin{tcolorbox}[colback=white, colframe=gray!60!black, title= DPO vs MPO comparison]
\begin{tcolorbox}[colback=white, colframe=black!30, title=\textbf{User Query}]
\textit{“Provide a pleasant compliment based on the quality you have been given. Truthfulness”}
\end{tcolorbox}
\begin{tcolorbox}[colback=white, colframe=dpoColor, title=\textbf{DPO Response}]
You possess a rare and valuable quality of truthfulness, which shines through in everything you say and do. Your authenticity is refreshing and inspiring, and it's clear that you're committed to being honest and transparent in all your interactions. Your words carry weight because they're backed by integrity, and it's a true asset to be around someone who embodies such a high level of trustworthiness. Your truthfulness is a gift to those around you, and it's something to be deeply proud of.
\end{tcolorbox}
\begin{tcolorbox}[colback=white, colframe=mpoColor, title=\textbf{MPO Response}]
Here's a pleasant compliment:
"You have a rare gift for authenticity, and it's truly refreshing to be around someone as honest and transparent as you. Your words carry weight because you always speak from the heart, building trust and respect with everyone you interact with. Your truthfulness is a treasure, and it's a privilege to know someone as genuine and dependable as you!"
\end{tcolorbox}
\end{tcolorbox}

\clearpage

\begin{tcolorbox}[colback=white, colframe=gray!60!black, title= DPO vs MPO comparison]
\begin{tcolorbox}[colback=white, colframe=black!30, title=\textbf{User Query}]
\textit{“Write a to-do list based on the given information. Daily tasks in the pharmacy?”}
\end{tcolorbox}
% \vspace{0.5em}
\begin{tcolorbox}[colback=white, colframe=dpoColor, title=\textbf{DPO Response}]
Based on typical daily tasks in a pharmacy, here's a to-do list:

**Morning Routine**

1. Restock shelves with medications and supplies
2. Check and respond to phone calls and voicemails
3. Review and update patient medication profiles and records

**Patient Care**

1. Fill and dispense prescriptions accurately and efficiently
2. Answer patient questions and provide medication counseling
3. Monitor and manage medication therapy plans for chronic conditions

**Inventory Management**

1. Conduct inventory counts and report any discrepancies
2. Order new stock and supplies as needed
3. Manage expired or recalled medications and dispose of them properly

**Administrative Tasks**

1. Process and submit claims to insurance companies
2. Manage and reconcile cash and credit transactions
3. Update patient records and billing information

**End-of-Day Tasks**

1. Clean and sanitize work areas and equipment
2. Review and prepare for the next day's tasks
3. Document any issues or concerns that arose during the day

\end{tcolorbox}
% \vspace{-10.0em}
\begin{tcolorbox}[colback=white, colframe=mpoColor, title=\textbf{MPO Response}]
Here's a comprehensive to-do list for daily tasks in a pharmacy:

**Morning Routine (Before Opening)**

1. **Inventory Management**:
	* Check inventory levels of medications, supplies, and equipment.
	* Restock shelves and replenish stock as needed.
2. **Preparation**:
	* Prepare prescription labels, bags, and packaging materials.
	* Set up cash register, credit card machines, and point-of-sale systems.
3. **Computer Systems**:
	* Log in to pharmacy management software (e.g., Meditech, Cerner) and update patient records.
	* Ensure all systems are functioning properly.

**During Business Hours**

1. **Fill Prescriptions**:
	* Process and fill patient prescriptions accurately and efficiently.
	* Verify prescriptions with physicians (if necessary) and communicate with patients.
2. **Customer Service**:
	* Assist customers with questions, concerns, or medication inquiries.
	* Provide guidance on medication usage, side effects, and interactions.
3. **Order Entry**:
	* Enter new prescriptions, refill requests, and medication orders into the system.
	* Manage backorders and special orders.
4. **Dispensing and Packaging**:
	* Package and label medications, including controlled substances.
	* Ensure proper storage and security of controlled substances.
5. **Administrative Tasks**:
	* Manage phone calls, emails, and patient inquiries.
	* Update patient records and maintain accurate records.

**End of Day/Before Closing**

1. **Closing Procedures**:
	* Count and reconcile cash, credit card transactions, and inventory.
	* Secure the pharmacy, including controlled substances and sensitive areas.
2. **Reporting and Documentation**:
	* Complete daily reports (e.g., prescription fills, inventory, and sales).
	* Update records and submit reports as required by regulatory authorities.
3. **Cleanliness and Organization**:
	* Maintain a clean and organized pharmacy environment, including workstations and storage areas.
\end{tcolorbox}
\end{tcolorbox}

\begin{tcolorbox}[colback=white, colframe=gray!60!black, title= DPO vs MPO comparison]
\begin{tcolorbox}[colback=white, colframe=black!30, title=\textbf{User Query}]
\textit{“GIVE ME AN EXECUTIVE SUMMARY OF LOTTADIGITAL.COM?”}
\end{tcolorbox}
% \vspace{0.5em}
\begin{tcolorbox}[colback=white, colframe=dpoColor, title=\textbf{DPO Response}]
Based on my research, LotT Digital is a digital marketing agency that specializes in providing comprehensive solutions to help businesses grow online. Here's an executive summary:

**Key Focus Areas:**
1. Digital Strategy: LotT Digital develops tailored strategies for clients, encompassing market analysis, competitor research, and goal-setting.
2. Web Development: The agency builds custom websites, e-commerce platforms, and web applications using modern technologies like WordPress, Magento, and React.
3. Digital Marketing: Services include SEO, PPC advertising (Google Ads), social media marketing, email marketing, and content creation to drive traffic and conversions.
4. Data Analytics: LotT Digital provides insights and reporting to measure campaign performance, track key metrics, and inform data-driven decision-making.

**Target Industries and Clients:**
1. Mid-to-large-sized businesses across various sectors, including e-commerce, finance, healthcare, and professional services.
2. Companies seeking to enhance their online presence, increase website traffic, and drive revenue growth.

**Unique Value Proposition:**

LotT Digital's team of experts takes a holistic approach to digital marketing, combining strategy, creativity, and technical expertise to deliver measurable results. They prioritize transparency, regular communication, and continuous optimization to ensure client success.

\end{tcolorbox}
% \vspace{-10.0em}

\begin{tcolorbox}[colback=white, colframe=mpoColor, title=\textbf{MPO Response}]
Here's an executive summary of Lottadigital.com:

**Company Overview: Lottadigital**

Lottadigital.com is a leading digital marketing agency specializing in E-commerce Solutions, Digital Strategy, and Performance Marketing. Headquartered in the UK, the company helps online businesses thrive in a competitive landscape.

**Key Focus Areas:**

1. **E-commerce Services**: Lottadigital offers comprehensive e-commerce solutions, including website design, development, and optimization for various platforms (e.g., Shopify, Magento, BigCommerce).
2. **Digital Strategy**: They craft tailored strategies to drive business growth, leveraging expertise in SEO, PPC (Google Ads), social media marketing, and content marketing.
3. **Performance Marketing**: The agency excels in performance-based marketing, focusing on data-driven campaigns that deliver measurable ROI, such as affiliate marketing, influencer partnerships, and retargeting.
4. **Specializations**: Lottadigital has expertise in niches like fashion, beauty, health, and lifestyle, with a strong understanding of their respective markets and consumer behaviors.

**Key Benefits:**

* Data-driven approach to optimize online presence and conversions
* Customized solutions for businesses seeking to improve their digital footprint
* Proven track record of increasing website traffic, sales, and revenue growth
* Expertise in staying up-to-date with the latest industry trends and platforms

**In summary**, Lottadigital.com is a full-service digital marketing agency that empowers e-commerce businesses to succeed by providing strategic, performance-driven solutions, tailored to their specific needs. Their expertise helps clients achieve measurable success and stay ahead in a rapidly evolving digital landscape.
\end{tcolorbox}
\end{tcolorbox}

\end{document}